\documentclass[10pt,journal,compsoc]{IEEEtran}
\ifCLASSOPTIONcompsoc
\usepackage[nocompress]{cite}
\else
\usepackage{cite}
\fi

\usepackage[pdftex]{graphicx}
\graphicspath{{figure/}}
\usepackage[dvipsnames]{xcolor}
\usepackage{tikz}
\usepackage[normalem]{ulem}
\usepackage[T1]{fontenc}
\usepackage{amsmath}
\usepackage{mathtools}
\usepackage{amsfonts}
\usepackage{bm}
\usepackage{graphicx,subfigure}
\usepackage{caption}
\usepackage{xcolor,colortbl}
\usepackage{siunitx}
\usepackage{url}
\usepackage{dsfont}
\usepackage{url}
\usepackage{enumitem}

\usepackage{algorithm}
\usepackage{algpseudocode}
\usepackage{multirow}
\usepackage{verbatim}
\usepackage{amssymb}
\usepackage{soul}
\usepackage{,amsfonts,amsthm,bm}

\usepackage{array}

\newcolumntype{C}[1]{>{\centering\arraybackslash}m{#1}}
\newcolumntype{L}[1]{>{\arraybackslash}m{#1}}

\colorlet{mygreen}{green!60!gray}

\newif\ifcomments

\commentstrue % uncomment to display comments in text
\ifcomments

    \newcommand{\BTcomm}[1]{\textcolor{mygreen}{{#1}}}

    \newcommand{\BT}[1]{\textcolor{blue}{{#1}}}
    \newcommand{\MB}[1]{\textcolor{magenta}{{#1}}}

    \newcommand{\TODO}[1]{\textcolor{red}{{#1}}}

\else
    \newcommand{\BTcomm}[1]{}
    \newcommand{\BT}[1]{}
    \newcommand{\MB}[1]{}
    
    \newcommand{\TODO}[1]{}
\fi

\newtheorem{theorem}{Theorem}

\begin{document}

\author{\IEEEauthorblockN{
		Benedetta Tondi,~\IEEEmembership{Member,~IEEE},
		Andrea~Costanzo,
		Mauro Barni,~\IEEEmembership{Fellow,~IEEE}
	}
\thanks{Authors  are with the Department of  Information Engineering and Mathematics, University of Siena, 53100 Siena, ITALY.}}

%\title{Chameleon Watermark: A Robust  DNN Watermarking Algorithm with Invisible Watermark}

\title{Robust and Large-Payload DNN Watermarking via Fixed, Distribution-Optimized, Weights}

\IEEEtitleabstractindextext{
\begin{abstract}
%
%We propose a white-box, multi-bit
%watermarking method that can achieve large payload and improved robustness with respect
%to existing algorithms.
The design of an effective multi-bit watermarking algorithm hinges upon finding a good trade-off
between the three  fundamental requirements forming the watermarking trade-off triangle, namely, robustness against network modifications, payload, and unobtrusiveness, ensuring minimal impact on the performance of the watermarked network.
In this paper, we first revisit the nature of the watermarking trade-off triangle for the DNN case, then we exploit our findings to propose a  white-box, multi-bit watermarking method achieving very large payload and strong robustness against network modification.
In the proposed system, the weights hosting the watermark are set prior to training, making sure that their amplitude is large enough to bear the target payload and survive network modifications, notably retraining,  and are left unchanged throughout the training process. The
distribution of the weights carrying the watermark is theoretically optimised to ensure the secrecy of the watermark and  make sure that the watermarked weights are indistinguishable from the non-watermarked ones.
The proposed method can achieve outstanding performance, with no significant impact on
network accuracy, including robustness
against network modifications, retraining and
transfer learning,
while ensuring a  payload which is
out of reach of state of the art methods achieving a lower - or at most comparable - robustness.
\end{abstract}
\begin{IEEEkeywords}
DNN IPR protection, Deep Learning security, DNN watermarking, White box watermarking, Robust DNN watermarking
\end{IEEEkeywords}
}
%\MB{Prima della fine dobbiamo cambiare il titolo: information coding e' troppo generico significa tutto e niente} \BTcomm{Modificato di nuovo. Fors emeglio discuterne insieme.}

\maketitle
\IEEEdisplaynontitleabstractindextext
\IEEEpeerreviewmaketitle

%\BTcomm{Altro possibile titolo,'Chameleon Watermark: A Robust DNN Watermarking Algorithm ensuring  Watermarked Weights Indistinguishability'. Il dubbio che ho è sull'uso del termine 'indistinguishable' che non mi convince del tutto (bognerebbe infatti dire indistinguibile da cosa). Forse possiamo dire che e' implicito..... Ma alla fine anche invisibile potrebbe andare no?}
%
%\MB{Usare tre volte la parole watermark nel titolo e' proprio brutto. Propongo due alternative (io preferisco la seconda):\\~\\
%Chameleon Watermark: Robust and High-Payload DNN Watermarking with Improved Undetectability,\\~\\
%Chameleon Watermarking: Hiding a Robust High-Payload Watermark within the weights of a DNN model}

\section{Introduction}

Neural network watermarking has been proposed as a solution to protect the Intellectual Property Rights (IPR) associated to DNN (Deep Neural Networks) models. In particular, the presence of a proprietary watermark within a DNN model has been proposed as a way to prove the ownership of the model, thus making it possible to trace possible abuses and resolve ownership disputes \cite{adi_turning_2018}. In addition to ownership verification, DNN watermarking can be used to trace back the history of DNN models, for fingerprinting and traitor tracing \cite{chen2019deepmarks}, to verify the integrity of a model \cite{he2018verideep}, and even to determine the source of contents generated by generative models \cite{ong2021protecting}. For most applications it is required that the watermark be robust against DNN modifications, like pruning, quantisation, fine tuning and even transfer learning.

As it is well known from previous studies on digital media watermarking \cite{craver98}, and as it has been restated recently in the framework of DNN watermarking \cite{NEURIPS2019}, the mere presence of a watermark within a DNN model is not enough to doubtlessly prove the ownership of the model. Such a possibility, in fact, depends on the ownership verification protocol wherein the watermark is used, and the threat model associated to it (see for instance \cite{adi_turning_2018, LWL22}). In some cases, the watermark needs to be authenticated by a certification authority, while in other cases additional properties such as non-invertibility must be satisfied, or the possibility of  embedding a minimum payload guaranteed. Similar arguments hold for applications other than ownership verification \cite{Park23}.

In order to investigate the very essence of the watermarking process, that is the process whereby a certain information is indissolubly tied to a DNN model, here we avoid to link our study to a specific application, rather we consider watermarking as a primitive defined by its black-box characteristics and by certain general properties like robustness, payload and unobtrusiveness. The way the watermark is used within a protocol, like for instance an ownership verification protocol, and the possible additional requirements stemming from the protocol, are left to other researches.

In the above framework a first distinction must be made between zero-bit and multi-bit watermarking. With zero-bit watermarking, the watermark extractor, now called detector, is only asked to decide wether the inspected model contains a given known watermark or not. With multi-bit watermarking, instead, the watermark extractor, more properly referred to as decoder, extracts the watermark bits without knowing them in advance. In both cases, the extraction of the watermark requires the knowledge of a secret key\footnote{In some zero-bit watermarking schemes, the key corresponds to the watermark itself, however, the watermark and the key play a different role and it is advisable to keep them separate.}, whose presence is crucial to ensure the secrecy of the watermark, and avoid that non-authorised users can access the watermarking channel\footnote{For sake of simplicity, here we broadly define the watermarking channel as the mechanism whereby the watermark payload is associated to the host model. At the same time, access to the watermarking channel is defined as the possibility to write, read, or erase the watermark bits from the host network (see \cite{cayre2005watermarking,kalker2001considerations} for a more rigorous and comprehensive description of these concepts).}
% \BTcomm{Che articolo di Furon?}) \MB{This one: F. Cayre, C. Fontaine and T. Furon, "Watermarking security: theory and practice," in IEEE Transactions on Signal Processing, vol. 53, no. 10, pp. 3976-3987, Oct. 2005, doi: 10.1109/TSP.2005.855418. however the following paper by T. Kalker - T. Kalker, "Considerations on watermarking security," 2001 IEEE Fourth Workshop on Multimedia Signal Processing Cannes, France, 2001, pp. 201-206, doi: 10.1109/MMSP.2001.962734 - is even more appropriate, we can cite both}.}.

A key-difference between multimedia and DNN watermarking derives from the observation that a DNN model is not a static object but a function defined by the way it maps the input samples into the output space \cite{barni2021dnn}. In the following, we indicate such a mapping as $y=\Phi(x,\bf{w})$, where $x \in \chi$ indicates the input of the model, $y$ the corresponding output, and $\bf{w}$ the network weights\footnote{In general $\bf{w}$ includes also the network biases, however, in our scheme the watermark is embedded only in the weights so, without loss of generality, we will refer to $\bf{w}$ as the weights of the network.} (we sometimes refer to them as the network coefficients). In the simplest case (white-box watermarking), the watermark is extracted directly by looking at $\bf{w}$. In other cases, the model is accessible only in a black-box modality and hence the watermark is extracted by looking at the output of the network in correspondence to a set of specific inputs (black-box watermarking). In yet other cases, like for generative models, the output of the model has a large enough entropic content and hence the watermark can be retrieved from any output $y$ even without knowing the corresponding input (box-free watermarking). Examples of white-box, black-box and free-box DNN watermarking algorithms are described, respectively, in \cite{uchida2017embedding}, \cite{rouhani2019deepsigns} and \cite{fei2022supervised}.

In this paper, we focus explicitly on white-box, multi-bit watermarking. In this case, the design of an effective watermarking algorithm boils down to finding a good trade-off between three general requirements (watermarking trade-off triangle), namely: {\em robustness} against network modifications, {\em payload}, measured as the number of bits the watermark consists of, and {\em unobtrusiveness}, that is the capacity of the watermarking technique to embed the desired payload without affecting the performance of the watermarked network.
In particular, our work stems from the recognition of the peculiar relationship existing between the above requirements in the case of white-box DNN watermarking.
We first argue that due to the particular way DNNs work, unobtrusiveness is linked {\em very weakly} to the robustness and payload requirements, unless an additional constraint regarding the secrecy of the watermark is considered. Failing to recognise the peculiarity of such a relationship results in watermarking schemes with suboptimal performance in terms of robustness, payload or both.
Then, we exploit the above insight to propose a new white-box, multi-bit watermarking algorithm with large payload and improved robustness with respect to existing algorithms. More specifically, in our scheme the values of the DNN weights hosting the watermark are fixed prior to training, making sure that their amplitude is large enough to bear the target payload and survive retraining and other network modifications. The watermarked weights are then frozen and are not updated during the learning phase. In this way, the
 other weights are determined in a watermark-dependent way,  co-operating with the
fixed weights to accomplish the network task.
 To ensure the secrecy of the watermark and avoid that the watermarked weights are easily identifiable, the distribution of the watermarked coefficients is optimised by minimizing the Kullback-Leibler (KL) distance between the watermarked and non-watermarked weights for a given amplitude of the watermark. As done in  other systems, robustness is further enhanced by spreading the watermark bits across several host coefficients.

We verified the effectiveness of the proposed watermarking algorithm by considering different architectures and classification tasks addressing different application domains, namely, image classification and image manipulation detection.
The results we got show that the proposed method can achieve good performance even for large payloads, with negligible impact on the accuracy of the underlying classification task. We also verified the strong robustness of the watermark against the most common network modifications, including, pruning, weights quantisation, and, most noticeably, retraining. Moreover, thanks to the optimization of the distribution of the watermarked weights, the marked coefficients are indistinguishable from the non-marked ones.

%With the above ideas in mind, the contribution of this paper can be summarised as follows

The rest of this paper is organised as follows: in Section \ref{sec.revisit}, we revisit the watermarking trade-off triangle and lay the basis of the new watermarking algorithm proposed in this paper. We also give a brief review of existing schemes by the light of the previous analysis.
In Section \ref{sec.model}, we introduce the notation used throughout the paper and state the requirements to be satisfied by the watermark. In Section \ref{sec.algorithm}, we describe the proposed watermark embedding and extraction algorithms. Section \ref{sec.methodology} and \ref{sec.experiments} are devoted, respectively, to the description of the experimental setting and to the discussion of the experimental results. Finally, in Section \ref{sec.conc}, we draw our conclusions and highlight directions for future work.

\section{The watermarking trade-off triangle revisited}
\label{sec.revisit}

In classical watermarking theory \cite{barni2004watermarking, Cox02}, the main goal of a multi-bit watermarking algorithm is to find a good, possibly optimum, trade-off between three opposite requirements summarised by the watermarking trade-off triangle shown in Figure \ref{fig.triangle}.

\begin{figure}
\centering\includegraphics[width=0.65\columnwidth]{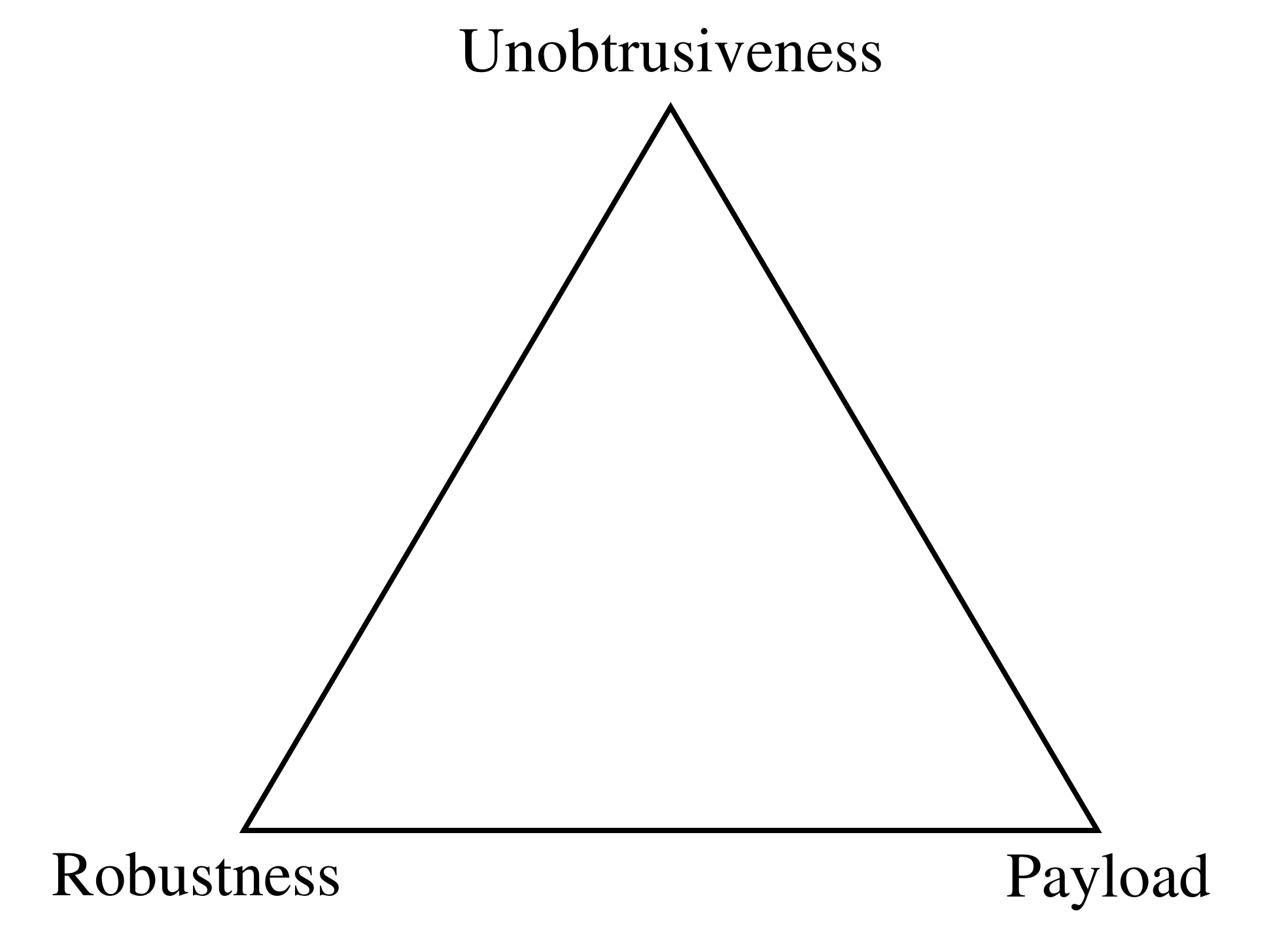}
\caption{The watermarking trade-off triangle.}
% \MB{Memo: remove security form the robustness corner}}
\label{fig.triangle}
\end{figure}

Payload and robustness have a similar meaning for media and DNN watermarking, even if the kind of manipulations the watermark should be robust to is completely different in the two cases. In media watermarking, common manipulations include lossy compression, filtering, resampling, cropping etc, while in the DNN case, the manipulations we are interested in are pruning, quantisation and, most noticeably, fine tuning and transfer learning. The unobtrusiveness requirement, however, plays a completely different role in the two cases. In media watermarking, unobtrusiveness requires that the {\em perceptual} quality of the watermarked media is not degraded due to the presence of the watermark. This, for instance,  means that the watermark should be {\em invisible} in the case of image watermarking, and {\em inaudible} in the case of audio watermarking. In media watermarking, then, unobtrusiveness has a direct impact on the amplitude or the {\em strength} of the watermark, usually measured as the distance (possibly the perceptual distance) between the original content and the watermarked one. In turn, this means that increasing the amplitude of the watermark to improve its robustness or increase its payload is possible only up to certain extent. In (white-box) DNN watermarking, the situation is completely different. Here the unobtrusiveness of the watermark requires that its presence has a negligible, if any, impact on the performance achieved by the network with regard to its intended task, hereafter referred to as primary task. This requirement is not linked directly to the amplitude of the watermark, even because in the DNN case an original set of non-watermarked coefficients to measure the distance from does not exist. In addition, the number of coefficients defining a network is usually highly redundant, thus making it possible, for instance, to prune many of them without any noticeable effect on the performance of the network. Eventually, the saturation effect enforced by the most commonly used activation functions limits the impact of very large intermediate values on the final output of the network. The striking conclusion of these observations is that, in principle, one could embed a very large payload by encoding the watermark in the amplitude of very few, in the limit even only one, extremely large coefficients, and freezing them during training. As with any communication channel \cite{CandT}, the absence of constraints on the amplitude, or to better say the power, of the coefficients bearing the watermark, results in unbounded capacity.

The obvious drawback of the approach outlined above is that the watermarked coefficients would be easily detectable, thus compromising irremediably the secrecy of the watermark. A possible way to alleviate this problem is to spread the watermark over many coefficients of limited amplitude, however, the number of coefficients that can be frozen during the training process is bounded by the unobtrusiveness constraint, since freezing a too large fraction of coefficients may degrade the performance of the network on the primary task. It is the presence of the secrecy requirement, then, that creates the trade-off between unobtrusiveness and the other two corners of the triangle, namely robustness and payload.
\begin{figure}
\centering\includegraphics[width=0.4\columnwidth]{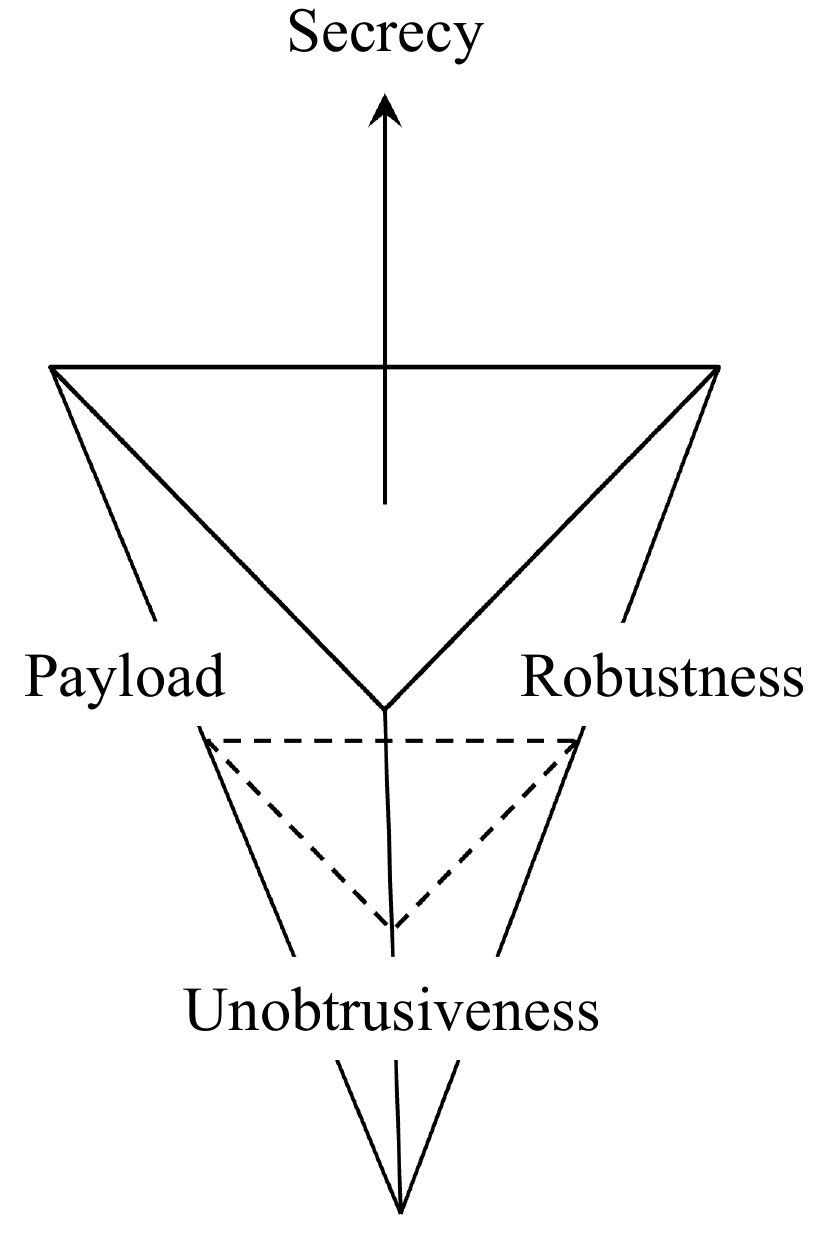}
\caption{DNN watermarking trade-off tetrahedron. }
\label{fig.tetra}
\end{figure}
With the above ideas in mind, the trade-off between the various requirements of DNN watermarking is better illustrated by a reverse tetrahedron rather than a triangle (see Figure \ref{fig.tetra}). When the secrecy requirement is relaxed, close to the vertex of the tetrahedron, the corners of the watermarking trade-off triangle get closer, thus making it easier to satisfy all of them simultaneously. The opposite situation occurs when secrecy is given more importance, in which case the trade-off triangle gets larger, thus creating a stronger tension between its corners.

In view of the previous ideas, it is evident that to design a robust and high payload multi-bit DNN watermarking algorithm, it is advisable to encode the watermark into the amplitude of large coefficients by paying attention to evaluate the impact that the amplitude of the watermarked coefficients has on watermark secrecy.

\subsection{Prior art}
\label{subsec.prior}

%\MB{In my view it is not necessary to specify {\em by the light \dots}, it is obvious that the prior art refers only to the themes the paper deals with}

On the basis of the framework provided by the previous discussion and summarised in Figure \ref{fig.tetra}, it is easy to realize that most white-box watermarking techniques proposed so far fail to address the watermarking trade-off properly.
According to the paradigm adopted by most of the methods proposed so far, in fact, watermarking is achieved by adding a watermarking loss term to the loss function used during training. Then, watermarking is carried out simultaneously with training, by letting the weights of the network be the solution of the following optimization problem:
\begin{equation}
{\bf w} = \arg \min (\mathcal{L }+ \lambda \mathcal{L }_w),
\label{eq.losswat}
\end{equation}
where $\mathcal{L }$ is the primary loss, whose minimization ensures that the performance of the network are good, and $\mathcal{L }_w$ guarantees that the watermark can be extracted without (too many) errors. $\mathcal{L }_w$ is usually related to the error probability of the watermark decoder, which, in most cases is designed based on a spread spectrum paradigm. In \cite{Uchida17}, for instance, and many subsequent works inspired to it, e.g., \cite{chen2019deepmarks, rouhani2019deepsigns, adamFernando, liu2021watermarking, yue21jins}, the watermark bits are encoded in the sign of the projection of the marked coefficients onto a set of $n$ pseudo-random sequences. The watermark loss $\mathcal{L }_w$, then, corresponds to the cross-entropy between the true watermark bits and the soft bit estimates obtained by applying a sigmoid function to the projections. In the above setting, the tradeoff between payload and robustness is achieved by using longer spreading sequences, and, most commonly, by adjusting the weighting parameter $\lambda$. It is clear, however, that increasing $\lambda$ ensures only that the cross-entropy term decreases, at the price of a stronger impact on the network performance, without any guarantee that this increases the robustness of the watermark.

A noticeable exception is the zero-bit watermarking algorithm described in  \cite{tartaglione2021delving}, where, similarly to our scheme, the weights hosting the watermark are fixed before training and are not modified during the training process. The strategy adopted to enhance robustness, however, does not rely on the amplitude of the marked coefficients. Rather, a
properly designed loss function explicitly making the watermarked weights more sensitive to loss
variations with respect to non-watermarked weights, is adopted, thus increasing the robustness of the watermark. In this way \cite{tartaglione2021delving} can
achieve a remarkable robustness against fine-tuning and weights quantization, however it still lacks robustness against transfer learning. In addition, no attention is given to the payload requirement,  given that the system proposed in \cite{tartaglione2021delving} belongs to the class of zero-bit watermarking methods.

A DNN watermarking algorithm that in some way achieves robustness by associating the watermark to large DNN coefficients is the one described in \cite{liu2021watermarking}. In this method, the watermark is still embedded by properly including a watermark loss term. However, the message bits are weighted by so called greedy residuals. Because of the way such greedy residuals are constructed, the watermark is indirectly embedded in the largest weights of the network, whose  amplitude is increased in order to get large residuals by multiplying the message bits with the correct sign. In addition, to facilitate obtaining large amplitudes, the watermark is embedded in the first layer of the network, where the weights tend to be larger. No particular attention is paid to secrecy, and hence the watermarked weights are easily identifiable, as discussed in Section \ref{subsec.SOA}.

In this work, we propose a simple but effective way to exploit the large amplitude of the watermarked coefficients to simultaneously achieve robustness and high payload. In particular, the watermark is spread over a number of large-amplitude, fixed coefficients (like in \cite{tartaglione2021delving}), while at the same time ensuring that the marked coefficients are as indistinguishable as possible from the non-marked ones. This provides a way to simultaneously trade-off between the 4 corners of the tetrahedron illustrated in Figure \ref{fig.tetra}. The experimental results shown in Section \ref{sec.experiments}, prove that in this way our system can achieve a remarkable robustness agains the most common network manipulations, including retraining for transfer learning, ensuring a payload which is out of reach of state of the art methods with comparable, or even lower, robustness.

\section{Notation and problem definition}
\label{sec.model}

%\TODO{Aggiusta riferimento sezioni!!!}

%\section{System Architecture}
%\section{Copy-Move Source/Target Disambiguation}

%We describe the proposed static multi-bit watermarking system with improved robustness against network modifications, and in particular transfer-learning.
%, and model compression.

%In the following, we first introduce the
%watermarking model and discuss the general requirements that a watermarking system
%has to satisfy.
% (Section \ref{sec.model}).

Let $\Phi$ indicate a generic non-protected DNN model, that is, a model trained for a given task without the watermark. The notation $\Phi^m$ is used to denote a watermarked model carrying out the same task.
We denote with $\mathcal{L}$  the primary loss function of the network, usually corresponding to cross-entropy.
The average loss measured across all the samples of the labeled training set $(\mathcal{X},  \mathcal{Y})$ is compactly denoted with $\mathcal{L}(\mathcal{X},  \mathcal{Y}, \Phi)$ ($\mathcal{L}(\mathcal{X},  \mathcal{Y}, \Phi^m)$, for the watermarked model).

%
% \MB{The meaning of $\mathcal{L}$ is not clear. For sure it refers to the output  of the classifier, this should be made clear in the symbolism. I'd prefer to show explicitly that $\mathcal{L}$ depends on the DNN model, which in turn, depends on $\Theta$}, and

We indicate with $\mathbf{b} \in \{0,1\}^l$  the vector with the watermark
bits, (sometimes referred to as the watermark message), where $l$ denotes the number of bits the watermark consists of.

The network weights can be generically represented as tensors. For each convolutional layer $j$, the dimensionality of the tensor with the weights is determined by  the kernel size of the filters,
the depth of the input and the number of filters (3-dim tensor).
For notational simplicity, we flatten the weight tensor
%${\bf W}_j$
into a vector ${\bf w}_j$ (row vector) containing all the weights of layer $j$.
We generically indicate the vector with all the weights of the network (be them watermarked or not) as ${\bf w}$.
With the above notation,
embedding the watermark bits into the weights of the network corresponds to embedding the vector $\mathbf{b}$ into ${\bf w}$.
The final watermarked vector should be such that the network preserves its functionality.
%
%Given that unobtrusiveness can not be controlled directly by modifying the weights of the model, in DNN watermarking
%embedding is typically carried out
%contextually to training to make sure that the network solves its primary task.
%

We denote with
%${\bf h} = (h_1, h_2,...,h_n)$
$\Omega$, hereafter called {\em host} index set, the set with the indexes of the weights hosting the watermark, that is, the positions in ${\bf w}$ that are selected
to carry the watermark message.
%
%\footnote{For simplicity, we refer to the weights specified by the host indexes as host weights; however, the concept of host signal does not have the same meaning it has in media watermarking, given that in DNN watermarking the weights do not exist {\em per se} but are directly generated in such a way to host the watermark.}.
%
The cardinality of $\Omega$ is indicated by $n$ (usually $n \ge l$).
With this notation, the vector with the watermarked weights should be indicated as  ${\bf w}^m = (w_{i_1}, w_{i_2}, \cdots, w_{i_n}), i_k \in \Omega$, however, for sake of simplicity, we will indicate it as ${\bf w}^m = (w^m_1, w^m_2, \cdots, w^m_n)$.
%
%Then,  $(w_{i})_{i \in \Omega} = {\bf w}^m$.
%\BTcomm{I am not sure it is ok using this formalism  to denote a vector. For sure it is understandable. I can not find any easy formalism with the same meaning}.
%
We also introduce the vector with non-watermarked weights, indicated by $\bar{\bf w}^m$. With the notation introduce above, we have $\bar{\bf w}^m = (w_{j_1}, w_{j_2}, \cdots, w_{j_n}), j_k \notin \Omega$, more simplicitely written as $\bar{\bf w}^m = (\bar{w}^m_1, \bar{w}^m_2, \cdots, \bar{w}^m_{N-n})$, where $N$ is the total number of coefficients the model consists of.
We observe that the host indexes may be selected from only one layer  (single-layer embedding),
%in which case  ${\bf w}^m \subset {\bf w}_j$,
or multiple layers  (multi-layer embedding).
%For instance, in the case of 2-layers embedding, we have that ${\bf w}^m \subset [{\bf w}_j, {\bf w}_k]$  where $j$ and $k$ denote the {\em host} layers. \BTcomm{Essendo vettori la notazione e' un po' imprecisa....ma e' chiara}
%We find useful to denote with ${\bf h}_j$ (${\bf h}_k$) the set of watermark indexes in layer $j$ ($k$).
We find useful to denote with ${\Omega}_k$ the set of watermark indexes in layer $k$. By indicating with $N_k$ the number of weights in layer $k$, then,
the percentage of watermarked weights in the network, denoted with $p^m$,
is equal to  $p^m= |{\Omega}|/N \times 100$, while the percentage of watermarked weights in layer $k$ is $p^m_k= |{\Omega}_k|/N_k \times 100$.
%Obviously, $p^m_k = p^m$ for single-layer embedding \MB{This is wrong, because the denominator of the percentages changes}.
%
We let ${f}_{{\bf w}^m}$ and  $ {f}_{\bar{\bf w}^m}$ denote the marginal empirical distribution of the watermarked and non-watermarked weights, estimated, respectively, on ${\bf w}^m$ and $\bar{\bf w}^{m}$, that is, the distributions induced by the respective sequences, and let $E_{{\bf w}^m}$ (res. $E_{\bar{\bf w}^m}$) denote the empirical expectation computed over ${f}_{{\bf w}^m}$ (res. ${f}_{\bar{\bf w}^m}$).
%empirical probability distribution induced by the sequence
%
Finally, we denote with $\mathcal{D}$  the Kullback-Leibler (KL) distance between two continuous probability distributions $f$ and $g$ \cite{CandT}, defined as  $\mathcal{D}(f||g) = \int_{x \in \mathds{R}} f(x) \log ( f(x)/g(x))$\footnote{For discrete distributions the integral is replaced by a summation over the alphabet of $x$.}.

%%
%\begin{equation}
%    \mathcal{D}(f||g) = \sum_{x \in \mathds{R}} f(x) \log \frac{f(x)}{g(x)},
%\label{eq.KLdiv}
%\end{equation}

%According to the above notation,
%embedding the watermark bits into the weights of the network corresponds to embed the vector $\mathbf{b}$ into the vector ${\bf w}$.
%More specifically, we denote with ${\bf w}_m = (w_{m1}, \cdots, w_{mn})$ the subset of network weights that are selected
%to carry out the watermark information, hereinafter called {\em host}
%weights. ???? HOST INDEXES more than WEIGHTS???? \footnote{The concept of 'host' signal is borrowed from classical media watermarking, where it explicitly refers to
%the non-watermarked
%asset ????CITA FERNANDO E NOI} Hence, $\mathbf{b}$ is embedded into  ${\bf w}_m$. The length $n$ of vector ${\bf w}_m$ is the length of the embedded watermark. We denote with $\overline{\bf w}_m$ is the remaining  set of non-watermarked weights. We observe that the host weights may be selected from only one layer $j$ (single-layer embedding), in which case  ${\bf w}_m \subset {\bf w}^j$, or multiple layers. For instance, in the case of 2-layers embedding, letting $j$ and $k$ denote the {\em host} layers,  we have that ${\bf w}_m \subset [{\bf w}^j, {\bf w}^k]$. \BTcomm{Essendo vettori la notazione e' un po' imprecisa....ma e' chiara}

%\subsection{Watermarking model and requirements}
%\label{sec.model}

%The security model and the requirements that the watermark message must  satisfy are described below.

%\subsection{Watermarking model}

\subsection{Watermarking model and requirements}

According to the multi-bit approach adopted in this paper, the watermaker wants to embed a watermark message $\mathbf{b}$ into the weights of the model  $\Phi^m$.
To do so, he relies on a secret key $K  = \{\Omega, {\bf s} \}$, consisting of the information on the host layers and weights ($\Omega$), and the spreading
sequence ${\bf s}$ used to spread the message ${\bf b}$ over the host coefficients.
%
%Watermark extraction requires white-box access to the network, hence qualifying our scheme as a white-box watermarking algorithm \cite{LI2021171}.
%
In addition to the usual unobtrusiveness, robustness and payload requirements, the watermark must satisfy the {\em integrity} requirement, asking that in the absence of model modifications, the decoded watermark $\hat{\bf b}$ should be close (ideally equal) to $\bf b$, that is, in the absence of modifications the bit error rate should be small (ideally zero).

With regard to {\em secrecy}, we loosely define it as the impossibility for an attacker to estimate the secret key by observing the weights of the watermarked model. In particular, here we focus on the estimation $\Omega$. Estimating $\Omega$, in fact, would allow the attacker to erase the watermark or overwrite it (while to decode it the knowledge of ${\bf s}$ is also necessary).

For the robustness requirement, we consider retraining for fine tuning and in particular transfer learning, and model compression, namely, parameter pruning and weights quantization.

{\em \bf Retraining} is a typical modification applied to models.
%, hence the  this case represents typical modifications that models may undergone.
A trained model is further trained for some epochs on the same or a different task.
We speak about {\em fine-tuning} when the network model is retrained to solve the same task in the same domain (often, on the same dataset used for the original training or on a subset of it), for some additional epochs, possibly with a different learning rate.
We speak about {\em transfer learning} in a more general setting where knowledge is transferred across domains or tasks.
%%It will slightly modify a model that was initially trained to solve an original task, and adjust it for a new task (possibly related to the original one).
%In a {\em transfer learning} scenario, instead,  the model is retrained for a different task, for more epochs.
%%, to learn a new problem.
%This is the case when the model, trained to solve a given task,
Transfer learning is widely adopted in practice since it turns out that transferring the knowledge from different tasks is less computationally expensive than  training a new model from scratch.
By referring to the categorization adopted in \cite{pan2009survey}, in the {\em transductive transfer learning} setting, the source and target tasks are the same, while the source and target domains are different.
%(when many labeled data from the source domain are available, while few labeled data from the target domain are available,  the model can be first trained on the source domain and then trained some additional epochs on data belonging to the target domain).
In the {\em inductive transfer learning} setting, the target task is different from the source task, while the domain can be the same or not.
In this case, the model trained on the source task is used as a pre-trained solution to train a network on a new task, possibly related to the original one.
Retraining alters the weights of the model, the degree of alteration depending on the difference between the domain and task targeted by the retraining process and those of the  original model, and on the number of retraining epochs. Arguably, during transfer learning, the  weights are modified by a larger extent than during fine-tuning, thus achieving robustness in the former scenario can be much harder.
%For any  watermarking scheme to be of practical use,  it is necessary to make sure that the watermark is robust against model retraining.

{\em \bf Compression}. DNN models are often squeezed to deploy them into low power or computationally weak devices, and  to reduce the storage demand.
Typical methods for model compression are {\em network pruning} and {\em weight quantization}.
The former cuts
off network weights, whose  value is smaller than a threshold,
% \MB{Are you sure pruning works like this? in this way pruning would depend on the input, which does not seem to be correct},
the latter reduces the numerical precision of the model coefficients, converting them from a floating point representation to a lower-precision, typically integer, representation.
More specifically,  model quantization can be described as follows:
\begin{equation}\label{eq.quantization}
%\begin{aligned}
%\textbf{w}_q &= \bigg\lfloor \frac{\textbf{w}}{\delta} \bigg\rfloor,\\
%\delta &= \frac{| w^{M}-w^{L}|}{2^{n_b}},
%\end{aligned}	
\textbf{w}_q = \bigg\lfloor \frac{\textbf{w}}{\delta} \bigg\rfloor \times \delta,\quad \text{$\delta = 2 w_{max}/2^{n_b}$}, % \text{$\delta = \frac{| w^{M}-w^{L}|}{2^{n_b}}$},
\end{equation}
%
%\MB{I think it should be $\bigg\lfloor \frac{\textbf{w}}{\delta} \bigg\rfloor \times \delta$}
where
% $\textbf{w}$ indicates the original weights ($\textbf{w}_q$ the quantized ones),
$\textbf{w}_q$ indicates the quantized weights,
$w_{max}$ represents the maximum absolute value assumed by the weights, and $n_b$ indicates the number of bits required to represent the quantized values.
%For instance, if the weights are quantized from a float32 format to int8, $n_b$  is equal to 8.
%\BTcomm{say that they should be multiple of 8 !!!}

We observe that network retraining represents the most critical case of network re-use for existing DNN watermarking schemes. Although some DNN watermarking methods have been proposed that can achieve a certain degree of robustness
against fine-tuning  \cite{tartaglione2021delving}, to the best of our knowledge, no method can simultaneously achieve high-payload and robustness in a transfer learning scenario.

\section{The proposed DNN watermarking method}
\label{sec.algorithm}

By relying on the insights we got by revisiting the watermarking trade-off triangle, our approach to simultaneously achieve a high payload and robustness against network modifications consists in embedding the watermark into several large coefficients and freeze them during the training process. In doing so we must i) ensure that the presence of several, fixed, large coefficients does not prevent training the network (unobtrusiveness), ii) avoid that the large coefficients hosting the watermark are easily identified and attacked (secrecy). In particular, the proposed watermark embedding algorithm works as  follows:
\begin{itemize}
%  \item A subset of network parameters (weights) are selected to carry the watermark information. The secrecy of the selection is guaranteed by the secrecy of the secret key.
  \item The information message $\bf b$ is embedded  before training and
  %the embedded weights $\textbf{w}_m$ are frozen, that is
  the  weights hosting the watermark message are not updated during training.
  \item The message bits are encoded in $\textbf{w}^m$ via direct sequence spread spectrum watermarking \cite{barni2004watermarking}. As a consequence, the message  $\bf b$
      %is not embedded directly, but is spread over
      is spread over several host weights. Specifically, the $l$  bits of the message are used to modulate a pseudo-random sequence ${\bf s}$ of length $n \ge l$, where  $S =  n/l$ is the spreading factor.
%      \MB{I would remove the following paragraph.
%      In this way,  the information coding paradigm is exploited
%{\em (spread spectrum watermarking is nothing
%but a particular form of bit repetition, that is the most simple form of information coding\footnote{The main difference with respect to bit repetition %pertains security, as with spread spectrum watermarking the watermark can not be read without knowing the  spreading sequence used.})}}.
   The strength of the watermark is controlled by means of a parameter  $\gamma$.
  \item To ensure that the watermarked  weights are indistinguishable from the other weights, the distribution of the spreading sequence is chosen in such a way to be as close as possible to that of the non-watermarked weights.
  %At the same time the strength of the watermark should be enough to be robust and in particular resist to transfer learning.
  In particular, for a given $\gamma$, the optimum distribution of the marked coefficients is derived theoretically by minimizing the KL distance between the distribution of the watermarked and non-watermarked coefficients (see Section \ref{sec.optumimDistrib}).
  %The $\gamma$ parameter also affect the  robustness (the larger $\gamma$, the more the robustness). However, arguably, a too large $\gamma$ would %result in a easy-to-detect (visible) watermark.
%  \item {\em Both single-layer and multi-layer embedding is performed. The multi-layer embedding solution is preferable for very large payloads and spreading factors, in which case the number of watermarked weights can be large,  and then, in the single-layer embedding case, they would represent a large percentage of the weights of the layer, thus affecting the unobtrusiveness.} \TODO{Inseriscilo da qualche altra parte}
\end{itemize}

For a given payload of $l$ bits, the trade-off between unobtrusiveness and robustness
%(see the trade-off triangle)
is determined by  the spreading factor $S$, while $\gamma$ controls the robustness and the secrecy of the watermark, with a too large $\gamma$ resulting in an easy-to-detect  watermark.

The main steps of the watermark embedding algorithm are illustrated in Fig. \ref{fig.watermark_extraction} and described in the following.
%After the information coding step described above, the embedding of the watermark and the model training
%\MB{Per come descriviamo ora il nostro metodo, il collegamento tra la figura  e il testo e' dubbio. Secondo me la figura si potrebbe anche togliere, o comunque il rapporto tra testo e figura chiarito meglio.}
%\BT{@Andrea: Ci sta bene una figura. Ho messo un placeholder}
%The details of each block are discussed in the following.
%are performed as detailed in the following.

\begin{figure}[]
	\centering
	\includegraphics[width=\columnwidth]{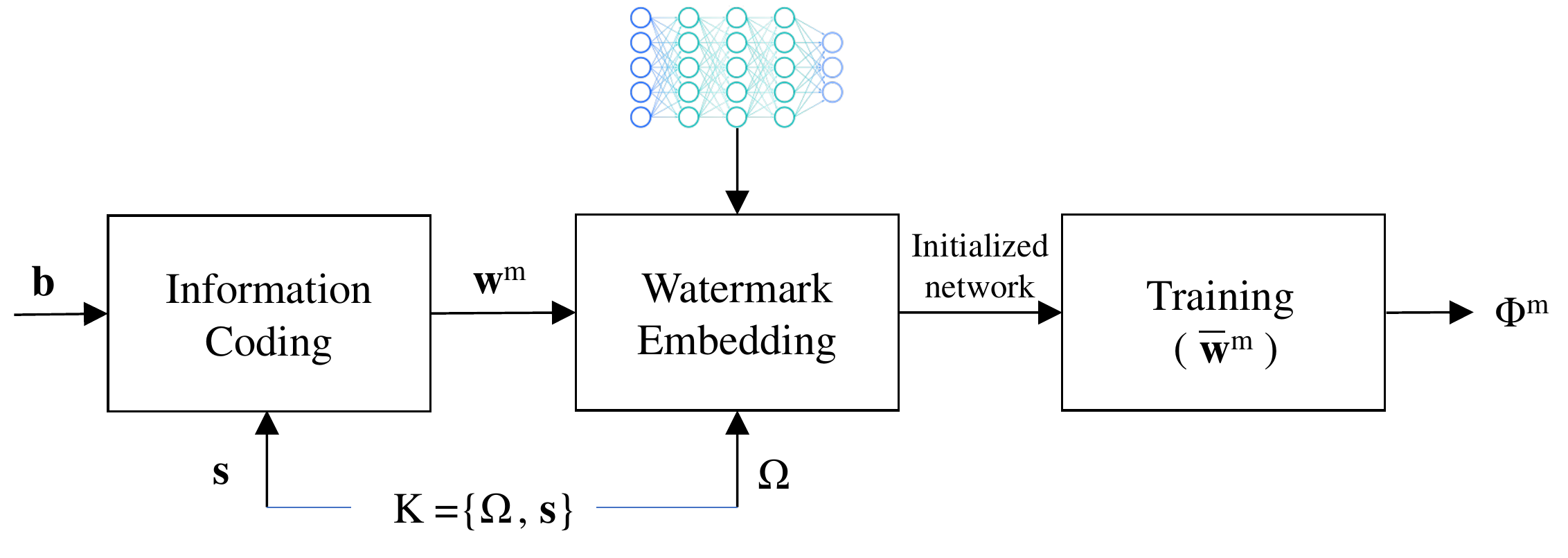}
	\caption{Watermark embedding procedure.} % \MB{Dobbiamo abbandolare il camaleonte} \MB{Alcuni font sono troppo piccoli}} %Proposed watermark embedding process.
	\label{fig.watermark_extraction}
\end{figure}

%{\em A similar idea can be found in Tartaglione paper ??, where, similarly to our method, the watermarked coefficients are set before the training procedure starts and are not modified during training. In particular, in ??, differently from our method, a zero-bit watermarking scheme is performed  by adopting a properly designed loss function explicitly thought to increase the robustness of the watermark  by ....... The method can achieve a remarkable robustness against fine tuning on a same dataset and weights quantization. However, the method can not achieve robustness against  transfer learning.}
%\BTcomm{NEL PRIPOR ART}

%To make the watermarked weights undistinguishable from the others
%(for security reasons), and by assuming that the weights of
%the to-be-watermarked architectures are initialised following
%a gaussian distributions , the watermarked
%weights are generated as follows......(the
%choice of the weights carrying the watermark is made based
%on a pseudorandom number generator distributing the watermark across the entire model), the watermarked weights
%are generated as:

%Main other ingredients: watermark spreading, multi-layer embedding.

%\MB{I do not like much the flow of the next sections. i think we should have only two subsections: i) the first describes the entire embedding and training procedure assuming that the spreading sequence is known, ii) the second optimizes the distribution of the training sequence. In this way you avoid speaking the training process in all the sections, thus simplifying the description} \BTcomm{Reorganized following your suggestion.}

\subsection{Watermark embedding}
\label{sec.encoding}

Let ${\bf u} \in \{-1,+1\}^l$ be the antipodal sequence associated to ${\bf b}$, that is $u_i = 1$  (res. $-1$) if $b_i = 1$ (res. $0$).
Before embedding, the vector
${\bf u}$
% is transformed into a watermark signal w = {wi,w2 ... wn} which is
%more suitable for embedding.
is  used to modulate
a  spread-spectrum sequence ${\bf s}$.
In particular,
%a pseudo-random sequence of length $n$ is generated and each $u_i$
each $u_i$  is used to modulate a block of coefficients of ${\bf s}$  (direct sequence spread spectrum  \cite{barni2004watermarking}).
%Eventually, b may be left
%as it is, thus leading to a scheme in which the watermark code is directly
%inserted within ....

Formally, given a spreading factor $S$, we generate a secret pseudo-random sequence ${\bf s} \in \mathds{R}^n$ of length $n = S \cdot l$ (the generation of ${\bf s}$ is discussed in Section \ref{sec.optumimDistrib}). Then, for each $i \in [1, l]$,  $u_i$ is used to modulate the sign of the elements of ${\bf s}$ corresponding to the indexes from $(i-1)S + 1$ to $iS$.
%weights corresponding to the indexes $(h_{S(i-1) + 1}, \cdots, h_{2S})$.
Specifically, the vector of watermarked coefficients  ${\bf w}^m$, at the output of the information coding block in Fig. \ref{fig.watermark_extraction}, is obtained as follows:
\begin{equation}
\label{eq.embedding}
w^m_{j} = u_i \cdot s_{j}, \quad  j \in \{(i-1) S+ 1, \cdots, i S\}, \quad i \in [1, l].
\end{equation}
Then, $(w^m_{(i-1)S + 1},  \cdots, w^m_{i S})$ is the vector of the weights associated to $b_i$.
In the embedding phase, the network weights in the positions indicated by $\Omega$
are initialized to the values given by   ${\bf w}^m$ (watermark embedding block in Fig. \ref{fig.watermark_extraction}).
With regard to the coefficients hosting the watermark, they are chosen at random according to the secret key $K$.
In addition to security, a further advantage of choosing the positions of the watermarked weights randomly is that the weights associated to the same bits are more likely to undergo independent changes during subsequent network modifications.

The network, then, is trained as follows: the network weights are initialized randomly, thus getting the initial vector of the weights ${\bf w}^{(0)}$. The
weights in the positions indicated by $\Omega$
are set to the
values given by ${\bf w}^{m}$. At this point, the  network is trained as usual, by minimizing the loss  across all the samples of the training set.
The  watermarked weights ${\bf w}^{m}$ are frozen during the training process, and only the non-watermarked weights ${\bar{\bf w}}^{m}$ are updated through backpropagation.
% through the backpropagation algorithm.

As we said, we must ensure that the  watermarked weights are as indistinguishable as possible from the non-watermarked ones, in order to avoid that they can be easily identified and attacked.
To this purpose, the distribution of ${\bf s}$, and, as a consequence, the distribution of  ${\bf w}^m$, is chosen in such a way to minimize the distinguishability between watermarked and non-watermarked weights, as described in the following section.

\subsection{Optimization of watermarked weights distribution}
\label{sec.optumimDistrib}

% \MB{This part does not work. In our system we are optimizing the distribution $f_{w^m}$, not directly ${\bm{w}}^m$.} \BTcomm{i tried fixing this by explaining that we first find convenient to formulate the problem as an optimization over the sequence and then pass to the optimization over the sequence distribution under mild assumption }

%As we said, a challenge of the proposed method is to make the watermarked weights as indistinguishable as possible from the watermarked weights, in order to avoid that such weights can be easily identified and then attacked.
%To this purpose, the distribution of ${\bf x}$, and then, as a consequence, the distribution of the watermarked weights ${\bf w}^m$, is chosen in such a way to minimize the distinguishability between watermarked and non-watermarked weights, as described in the following section.

%As we said, a challenge of the proposed method is to make the watermarked weights as indistinguishable as possible from the watermarked weights, in order to avoid that such weights can be easily identified and then attacked.

In this section, we address the problem of finding the  optimum distribution to generate the pseudo-random spreading sequence ${\bf s}$.
Although our final goal is to optimize the distribution of ${\bf s}$ (and as a consequence the distribution of ${\bf{w}}^m$),  we find  convenient to initially formulate the optimization  problem as a minimization carried out  over the weights ${\bf{w}}^m$.  Then, we will see that,  under some mild assumptions, we can rephrase the optimization over the weights as an optimization over the weights distribution ${f}_{{\bm w}^m}$.
% \MB{Notation problem: throughout the paper we treat $f$ as a scalar pdf, while writing ${f}_{{\bm w}^m}$ with a bold subscript hints to a multivariate pdf. In fact, we can do that since we assume that the sequence ${\bm w}^m$ is i.i.d.. Perhaps we should be more precise on this point.}

%Then, under some (mild) assumptions, we will be able to rephrase the optimization in function of the distribution.

The problem of training a DNN watermarked model can be formalized as follows\footnote{Rigorously speaking, the set of trainable parameters  also includes the biases, however, for sake of simplicity, we refer only to the weights since the watermark is not embedded in the biases.}:
 %
  %\MB{This formalization does not work well: you have to minimizations and two constraints without specifying the relationship between them}:
%%
%\begin{equation}\label{problem}
%	\begin{aligned}
%%		\min_{{\bar{{\bm w}}}^m, \beta}
%%& \hspace{0.2cm} \mathcal{L}(\mathcal{X}, \mathcal{Y}, \{{\bm w}^m, \bar{\bm w}^{m}, \beta\})\\
%%\min_{\Theta \backslash {\bm{w}}^m}
%\min_{\bar{\Theta}^m}
%& \hspace{0.2cm} \mathcal{L}(\mathcal{X}, \mathcal{Y}, \Theta ) \\ % \mathcal{L}(\mathcal{X}, \mathcal{Y}, \{{\bar{\Theta}}^m, {\bm \theta}_m \}) \\
%		\min_{{\bm{w}}^m} &\hspace{0.2cm} \mathcal{D}({f}_{{\bm w}^m} || {f}_{\bar{\bm w}^m})\\ % t\neq \phi(f(x))\\
%		& E_{{\bm w}^m}[w|w > 0] = \gamma, \quad E_{{\bm w}^m}[w|w < 0] = - \gamma
%%		\min &\hspace{0.2cm} ||\delta||_2 \\
%%		\textrm{s.t.} &\hspace{0.2cm}  \phi(f(x+\delta)) = t\\ % t\neq \phi(f(x))\\
%%		& x+\delta \in [0,1]^m
%	\end{aligned}
%\end{equation}
%%
%
\begin{equation}\label{problem}
	\begin{aligned}
%		\min_{{\bar{{\bm w}}}^m, \beta}
%& \hspace{0.2cm} \mathcal{L}(\mathcal{X}, \mathcal{Y}, \{{\bm w}^m, \bar{\bm w}^{m}, \beta\})\\
%\min_{\Theta \backslash {\bm{w}}^m}
\min_{\bar{\bf w}^{m}}
& \hspace{0.2cm} \mathcal{L}(\mathcal{X}, \mathcal{Y}, \bar{\bf w}^{m} \cup {\bf w}^{m,*} ) \\ % \mathcal{L}(\mathcal{X}, \mathcal{Y}, \{{\bar{\Theta}}^m, {\bm \theta}_m \}) \\
	\end{aligned}
\end{equation}
where, in order to minimize the distinguishability between watermarked and non-watermarked weights,  ${\bf w}^{m,*}$ is obtained  by solving the following optimization:
\begin{equation}\label{problem2}
	\begin{aligned}
		{\bf w}^{m,*} = \underset{{\bf{w}}^m}{\arg \min} \hspace{0.2cm} \mathcal{D}({f}_{{\bm w}^m} || {f}_{\bar{\bm w}^m}),
	\end{aligned}
\end{equation}
which also depends on $\bar{\bf w}^m$, with the additional constrain that
\begin{equation}\label{problem2-bis}
	\begin{aligned}
		E_{{\bm w}^m}[w|w > 0] = \gamma, \quad E_{{\bm w}^m}[w|w < 0] = - \gamma,
	\end{aligned}
\end{equation}
to control the average amplitude of the weights bearing the watermark.

In the following, we require that the watermarked weights distribution is symmetric (given that the distribution of the non-watermarked weights is typically symmetric, this goes w.l.o.g.). With this assumption, in fact,  the  watermarked weights ${w}^m_j$ obtained via  Eq. \eqref{eq.embedding} follows the same distribution of the pseudo-random sequence ${\bf s}$, regardless of ${\bf u}$.
Because of the symmetry of the distribution, we can rewrite equation \eqref{problem2-bis} as $E_{{\bm w}^m}[|w|] = \gamma$.
A large $\gamma$ results in a strong - hence expectedly more robust - watermark. In addition to robustness, $\gamma$ affects the secrecy of the watermark. A very large value of $\gamma$, i.e., such that $E_{\bar{{\bm w}}^m}[|w|] \ll \gamma$,
%affects the security, since it would make
would, in fact, make the presence of the watermark easily detectable upon inspection of the values of the weights.
%, as arguably $E_{\bar{{\bm w}}^m}[w] \ll \gamma$.

To go on, we observe that \eqref{problem} and \eqref{problem2} are entangled equations due to the presence of ${\bf w}^m$ and $\bar{\bf w}^m$ in both of them, so they can not be solved easily.
However, under the assumption  that the presence of the watermark does not affect the statistical distribution of the non-watermarked weights\footnote{We checked experimentally that the distribution of the weights corresponding to the non-host indexes of the non-watermarked model  and that of the watermarked model  are approximately the same, for all the watermark settings considered in this paper. More details are reported in Section \ref{sec.histograms}.},
%\BT{@Andrea: corretto?? abbiamo delle figure di distribuzione della rete marchiata e non marchiata dove si evince questo comportamento? forse bisognerebbe andare a vedere solo il comportamento del layer marchiato}},
the two minimizations
% over ${\bm \theta}_m$ and ${\bm{\bar{\theta}}}_m$,
can be separated as follows:
% and solved in cascade:
%olving the folliowing cascade
%
\begin{equation}\label{problem-sep}
	\begin{aligned}
		{\bf w}^{m,*} = \arg\min_{{\bf{w}}^m} &\hspace{0.2cm}\mathcal{D}({f}_{{\bm w}^m} || \tilde{{f}}_{\bar{\bm w}^m})\\ % t\neq \phi(f(x))\\
		& E_{{\bm w}^m}[|w|] = \gamma,
	\end{aligned}
\end{equation}
%
%\MB{The symbolism $E_{{\bm w}^m}[|w|]$ is not correct. To be rigorous we should say that the expected value is taken by considering the empirical %distribution of $w$, perhaps we could remove the subscript to avoid complicating the notation too much}
%and
%
\begin{equation}\label{problem-sep2}
	\begin{aligned}
%		 \{\bar{\bm w}^{m,*}, \beta^*\}= {\arg\min}_{\{\bar{\bm w}^{m}, \beta\}} &\hspace{0.2cm} \mathcal{L}(\mathcal{X}, \mathcal{Y}, \{{\bm w}^{m,*},
%\bar{\bm w}^{m}, \beta\}),
		\bar{\bf w}^{m,*} = \arg\min_{\bar{\bf w}^{m}} &\hspace{0.2cm} \mathcal{L}(\mathcal{X}, \mathcal{Y}, \bar{\bf w}^{m} \cup {\bf w}^{m,*}),
%\{{\bm{\bar{\theta}}}_m, {\bm \theta}_m^*\}). \\
	\end{aligned}
\end{equation}
%
%We now focus on  the solution of the first problem.
where $\tilde{f}_{\bar{\bm w}^m}$ is the distribution of the non-watermarked weights of a non-watermarked model $\Phi$ solving the same task.
%\footnote{The normal model is the one obtained by training the same network (in the same setting) without the watermark embedding.}
%
%
To a closer look, we observe that the minimization in \eqref{problem-sep} depends only on the probability density function of $\bar{\bf w}^m$ rather than on the specific sequence, so it can be conveniently reformulated as a minimization over the distribution ${f}_{{\bm w}^m}$:
\begin{equation}\label{problem-sep3}
	\begin{aligned}
		{f}_{{\bm w}^m}^* = \arg\min_{{f}_{{\bm w}^m}} &\hspace{0.2cm}\mathcal{D}({f}_{{\bm w}^m} || \tilde{{f}}_{\bar{\bm w}^m})\\ % t\neq \phi(f(x))\\
		& E[|w|] = \gamma.
	\end{aligned}
\end{equation}
Then, the sequence ${\bf w}^{m,*}$ appearing in \eqref{problem-sep} can be {\em any} sequence generated according to ${f}_{{\bm w}^m}^*$, thus allowing to keep the specific sequence ${{\bf w}^m}^*$ secret.
%\MB{In fact we should distinguish between ${{\bm w}^m}^*$ and ${\bm u}$}.

%We notice that the problem in \eqref{problem-sep} of determining the optimum distribution for a subset of weights, given a fixed distribution of the remaining set, has some ties with the problem solved in \cite{barni2018adversarial} for determining the optimum distribution of the adversarial samples for training set corruption.

Solving the problem in \eqref{problem-sep3} for a general distribution of the non-watermarked weights ${\bar{{\bf w}}}^m$ is not easy.
In the following, we solve it by assuming that the distribution of the non-watermarked weights  can be approximated by a Laplacian distribution (our experiments on several network architectures and tasks confirmed the goodness of this assumption, see Section \ref{sec.histograms}).
Let, then, $ \tilde{{f}}_{\bar{\bm w}^m}(w) = \frac{1}{2\lambda}e^{-|w |/\lambda}$, that is  $\bar{w}^m_i \sim \textrm{Laplace}(0,\lambda)$.
Under such an assumption, problem \eqref{problem-sep3} is solvable in closed form, as stated in the following theorem.
%\MB{Given that the theorem is a general one can we avoid the heavy symbolism $f_{\bm w}^m$ and write simply $f_w$?}
%
%
%
\begin{theorem}
Let $\mathcal{F}$ denote the set of symmetric distributions.
The minimization problem
%in \eqref{problem-sep-Laplace}
\begin{equation}\label{problem-sep-Laplace}
	\begin{aligned}
		\min_{f_w \in \mathcal{F}} &\hspace{0.2cm} \mathcal{D}\Big({f}_w \big|\big| \frac{1}{2\lambda}e^{-|w |/\lambda}\Big)\\ % t\neq \phi(f(x))\\
		& \textrm{s.t. } E[|w|] = \gamma
	\end{aligned},
\end{equation}
is equivalent to the problem of finding the {\em maximum entropy distribution} over all the symmetric probability density functions ${f}_{w}$ satisfying $E[|w|] = \gamma$, whose solution is $f_{w}^{*} = \frac{1}{2 \gamma} e^{-|w|/\gamma}$. Hence the optimum distribution for the watermarked weights is a Laplacian distribution with 0 mean and scale parameter $\gamma$, that is $w_i \sim \textrm{Laplace}(0, \gamma)$, for $i \in \Omega$.
%\MB{There is something wrong here given that for a $\textrm{Laplace}(0, 2 \gamma)$  $E[|w|] = 2 \gamma$, thus failing to satisfy the constraint in \eqref{problem-sep-Laplace}}.
\label{theo}
\end{theorem}

\begin{proof}
By observing that
\begin{align}
\mathcal{D}&\big( {f}_w|| \frac{1}{2\lambda}e^{-|w|/\lambda}\big) = \nonumber\\
   &\int_{w} {f}_{w}(w)  \bigg[ \log {f}_{w}(w)   -   \log \frac{1}{2\lambda}e^{-|w|/\lambda} \bigg] d w =  \nonumber\\
  &\int_{w} {f}_{w}(w)  \log {f}_{w}(w) d w - \log\frac{1}{2\lambda} + \frac{1}{\lambda}  \log e \cdot E[|w|],
\end{align}
and given that $E[|w|] = \gamma$, eq. \eqref{problem-sep-Laplace} can be rephrased as
\begin{equation}\label{problem-sep-rep-Laplace-2}
	\begin{aligned}
		\max_{{f}_{w}: E[|w|] = \gamma} &\hspace{0.2cm}  h({f}_{w}),
	\end{aligned}
\end{equation}
where $h(f) = - \int_{x}f(x) \log f(x) dx$ denotes the differential entropy.
Equation \eqref{problem-sep-rep-Laplace-2} is equivalent to finding the {\em maximum entropy distribution}  over all the probability density functions for which $E[|w|] = \gamma$  \cite{CandT}.

By rewriting the differential entropy integrating over the positive and negative supports separately, and by exploiting the symmetry of ${f}_{w}$, we get
\begin{align}
h(f_w) = & - \int_{\mathds{R}} {f}_{w}(w) \log {f}_{w}(w) d w= \nonumber\\
%= & \int_{0}^{\infty} f(\theta) \log f(\theta) d \theta - \int_{-{\infty}}^{0} f(\theta) \log f(\theta) d \theta
= & - 2 \int_{0}^{\infty} {f}_{w}(w) \log {f}_{w}(w) d w \nonumber\\
= & - \int_{0}^{\infty} {f}_{w}'(w) \log {f}_{w}'(w) d w + 1  \nonumber\\
= & \quad h( {f}_{w}') + 1,
\end{align}
where in the second-to-last equality we let:
\begin{align}
{f}_{w}'(w) =
\left\{\begin{array}{ll}
		2 {f}_{w}(w) & w \ge 0\\
		0  & w< 0\\ % t\neq \phi(f(x))\\
\end{array}\right..
\end{align}
%
%Therefore, maximizing the entropy of $\hat{f}_{{\theta}_m}$ is equivalent to maximizing the entropy of $\hat{f}_{{\theta}_m}'$
%
Therefore,
%the entropy-maximizing distribution in \eqref{problem-sep-rep-Laplace-2}  is equivalent to the entropy-maximizing
%  distribution of  the following
in order to find the entropy-maximizing distribution in \eqref{problem-sep-rep-Laplace-2},
we can equivalently solve the problem
% \BTcomm{per maggiore chiarezza andava riportata  la pdf in pedice a E. L'unico punto poco ambiguo e' questo, per cui ho preferito non appesantire troppo la notazione. \MB{We could remove the subscript in the expectations. The meaning seems to be clear all over the paper }}
%
%
\begin{equation}\label{problem-sep-rep-Laplace-2-equivalent}
	\begin{aligned}
%		 {{f}_{{\bm w}^m}}' = \arg\max_{{f}_{{\bm w}^m}'}
\max_{{f}_{w}'}
&\hspace{0.2cm}  h({f}_{w}')\\ % t\neq \phi(f(x))\\
		&  E[w] =  \gamma \\
        & {f}_w'(w) = 0,  \text{ for $w < 0$}.
	\end{aligned}
\end{equation}
%
%\MB{I found the error. Even if $f'$ is equale to twice $f$ on the positive axis, the constraint remains the same, so in the above equation we must set $E[w] = \gamma$}. \MB{I hope this does not mean that you will have to rerun all the experiments by generating the spreading sequence in the correct way}
The distribution we search for in \eqref{problem-sep-rep-Laplace-2-equivalent} is  the  maximum entropy distribution  over all the probability density functions with support $[0, \infty)$ satisfying $E[w] =  \gamma$.
%set where f (x) > 0 is called the support set of X
%
The solution of this problem is known (\cite{CandT}, Ex 12.2.5) and corresponds to ${f}_{w}' = \frac{1}{\gamma} e^{-|w|/\gamma}$, $w\ge 0$, hence yielding:
\begin{equation}
 {f}_{w}^*(w) = \frac{1}{2 \gamma} e^{-|w|/\gamma},  -\infty \le w \le \infty.
\end{equation}

%\BTcomm{Qui ho fatto riferimento a questo che nel Cover e Thomas viene riportato come esercizio..... Volendo si possono fare tutti i conti (che avevo fatto in prima battuta) per andare a risolvere la massimizzazione con il metodo di Lagrange. Cosi' facendo verrebbe una pagina con i conti (da spostare in Appendix).} \MB{Va bene cosi', se ce lo chiedono aggiungiamo un'appendice. Certo che e' curioso che il risultato non dipenda da $\lambda$, forse dovremmo commentare questo fatto.}
%
%
\end{proof}

Based on the result of Theorem 1, the watermarked weights must be generated following the distribution  $ {f}_{{\bm w}^m}^* = \frac{1}{2\gamma}e^{-|w |/\gamma}$, that is, ${\bf w}^{m,*} \sim \textrm{Laplace}(0, \gamma)$. To do so, we proceed as follows: we first generate the  pseudo-random sequence ${\bf s}$ according to a $\textrm{Laplace}(0, \gamma)$ distribution. Then, we apply Eq. \eqref{eq.embedding} to modulate the spreading sequence according to the watermark message, thus getting the vector of
watermarked weights  ${\bf w}^{m,*}$ still following a $\textrm{Laplace}(0, \gamma)$ distribution.

\subsection{Watermark extraction}
\label{sec.algorithm_extraction}

%The watermark extraction process is depicted in Fig. \ref{watermark_extraction}.
%As shown in the figure,
Watermark retrieval
%is carried out by inspecting the weights of $\Phi$ and
requires the knowledge of the
%of the key $K$, providing the
indexes of the watermarked weights
and the pseudo-random sequence ${\bf s}$.
%
%\begin{figure}[]
%	\centering
%	\includegraphics[scale=.4]{Fig3-extraction.pdf}
%	\caption{Watermark extraction process. \MB{This picture does not say much, I think it could be removed}}
%	\label{watermark_extraction}
%\end{figure}
%
The vector  ${\bf w}^m$ is first obtained by reading the weights of $\Phi^m$ in the positions indicated by $\Omega$, then the $i$-th bit of the watermark  is extracted as follows:
\begin{align}
\hat{b}_i =
\left\{\begin{array}{ll}
1 & \text{if $\sum_{j = (i-1)S + 1}^{iS}  s_j \cdot w^m_j \ge 0$} \\
0 & \text{otherwise} \end{array}.\right.
\end{align}

The Bit Error Rate (BER) is calculated as $\textrm{BER} = (\sum_{i} ({\bf b} \oplus  \hat{\bf b})_i / l)\times 100$, where  $\oplus$ denotes the bitwise XOR operation.
%\MB{Se vuoi lascia pure cosi', pero' a me piacerebbe di piu' usare le probabilita' al posto delle percentuali.}

We notice that, in the absence of modifications of the watermarked model, $\hat{\bf b}$ is equal to ${\bf b}$ by construction.
%  This is a peculiarity of many DNN watermarking \MB{I wouldn't say so, in schemes like Uchida's there's no guarantee that at the end of training the BER be equal to 0} algorithms and is a consequence of the fact that the desired behavior for the task and the watermark are learned
%simultaneously.
In classical watermarking theory, such a property is achieved by informed watermarking algorithms  \cite{barni2004watermarking}, whereby embedding is performed by applying a signal-dependent perturbation to the host signal.
Our DNN watermarking algorithm adopts a somewhat dual approach. The watermarked weights are first fixed, then the other weights are defined in a watermark-dependent way, in such a way that they {\em co-operate} with the fixed weights to accomplish the network task.
%\MB{I would stop here and remove the following parallelism with Dirty paper coding, which I think goes too far}
%{\em Drawing a parallelism with the informed embedding paradigm and the theory of writing in a dirty paper by Costa \cite{miller2004applying},
%the dirt is represented by the watermark signal, injected in a white paper, instead of by the host signal like in classical watermarking}

%\section{Proposed Method}
\section{Experimental methodology}
\label{sec.methodology}

We validated the proposed DNN watermarking technique by considering different architectures and tasks.

\subsection{Host networks and tasks}

We focused on tasks taken from two different application areas, that is, image forensics and pattern recognition.
%We consider several different tasks, from different applications areas:
Specifically, we considered the distinction of images generated  by Generative Adversarial Networks (GANs) from natural images \cite{goodfellow2020generative} (GAN detection), object recognition (CIFAR-10 and 100 classification  \cite{krizhevsky2009learning}) and traffic sign classification  (GTSRB classification \cite{Houben-IJCNN-2013}).
%
%%1) the problem of detecting Generative Adversarial Network (GAN) images ??PUT REFERENCE  ??? (GAN image detection); 2) the problem of traffic sign classification  (GTSRB
%%classification ???REFERENCE); 3) the problem of image classification, namely CIFAR classification ????REFERENCE.
%For the GAN image detector, we considered the problem of discriminating pristine face images from synthetic face images generated by the StyleGAN2 generative model \cite{Karras2019stylegan2}.
These tasks permit to test the effectiveness of our method in different scenarios, including
binary classification  and multi-class classification.
% (43 classes for GTSRB, and 10 classes for CIFAR-10).
%
For each task, we chose network architectures among those achieving state-of-the-art performance.
%
%
%With regard to the network architectures,

For the GAN detection task, we considered the discrimination of natural face images and images generated by the StyleGAN2 model \cite{Karras2019stylegan2}, by means of XceptionNet, which can achieve state-of-the-art performance \cite{gragnaniello2021gan}. We trained the network for 10 epochs with SGD optimizer, learning rate $0.01$ and batch size 32.
More details on network training and the dataset can be found in the authors' repository
\cite{GANnoGAN}.

For CIFAR-10 classification, we trained a ResNet18
% \cite{he2016deep} \BT{reference for this} \AC{Inserita.}
and a DenseNet169 architecture
\cite{huang2017densely},
following the parameter setting reported in \cite{DenseSettingCIFARnew},
% \BT{e' ancora il reference giusto? vale anche per ResNet?}\AC{ No, non era più giusto. L'ho aggiornato},
achieving the benchmark accuracy  for this task.
Specifically,
with both DenseNet and ResNet architecture, the network was trained for 200 epochs with SGD optimizer, learning rate $0.01$ with multi-step decay every 50 epochs and batch size 32.
In all our experiments, we consider augmentation.
CIFAR-100  and GTSRB are considered for the transfer-learning experiments.

%{\em For GTSRB and CIFAR-10 classification, we trained a DenseNet169 architecture \cite{huang2017densely}.
%Specifically, for GTSRB, the network was trained for 20 epochs with Adam optimizer, learning rate $0.001$ and batch size 64, reaching the benchmark accuracy  (https://benchmark.ini.rub.de/). }\TODO{Sposta a dopo - TL setting}
%

%For the GTSRB task, we trained a DenseNet169 architecture \cite{huang2017densely} following the parameter setting reported in ???,  reaching the benchmarks.
%
%For CIFAR-10 classification, we trained the network
%following the parameter setting reported in \cite{DenseSettingCIFAR},
%again achieving state-of the-art performance for this task.
%
%In all our experiments, we did not consider augmentation.

Notably, the architectures we used to assess the effectiveness of the algorithm are quite diverse.
They all have a similar number of parameters (in the order of $10^7$), however, they differ  in terms of depth
% (\AC{169} layers for DenseNet, 71 layers for XceptionNet \AC{and 18 for ResNet18}),
%
% for a similar number of parameters %\BT{????HOW MANY} times the number of parameters of XceptionNet \BT{AND THE DEPTH???})
and internal structure connections.
%\AC{Da un punto di vista concettuale, la profondità delle rete è quella che ti ho messo qui. Da un punto di vista puramente tecnico, la profondità in PyTorch (se stampi il numero di layers) è 603 per Dense, 136 per XCeption e 60 per Resnet. Questo succede perché un layer "teorico" è implementato con più layer "in pratica".}
While the block connections in the XceptionNet architecture are pretty standard, in DenseNet, there are dense blocks where  each convolutional layer is connected to every other
layer in a feed-forward fashion.
ResNet instead relies on residual blocks for feature extraction.

%Since the results of the experiments are similar for GTSRB and CIFAR-10, for sake of brevity,  we will provide an in-depth analysis of the GTSRB case, and give only a summary of the results we got on CIFAR-10 classification.
%
%\subsection{Parameters setting}

The DNN watermarking algorithm and network training have been implemented by using the \mbox{PyTorch 1.8} library and the code is made publicly
available for reproducibility (\url{https://github.com/andreacos/Deep-Neural-Networks-Robust-Watermark}).
% \BTcomm{@Andrea: Ricordarsi di aggiornare git con anche materiale/info supplementari al momento della submission}
%\BTcomm{@benedetta: ricorda di cambiare a pubblico}.

\subsection{Watermarking algorithm setting}%Setting of watermark parameters}

%\TODO{Selezione pesi da marchiare}

We tested the performance of our watermarking algorithm by considering both single-layer and multi-layer embedding. The multi-layer solution is necessary for large payloads and/or spreading factors, that is, for large $n$, when the percentage of watermarked weights in the watermarked layer for the single-layer embedding is too large,  thus deteriorating the network performance (as confirmed by our experiments).

In all the settings, we embedded the watermark in the convolutional layers from intermediate to deep.
All the information on the host layers $k$,  the total number of weights in the layer ($N_k$), and the  variance  of the distribution of the non-watermarked weights  for each  layer ($\sigma_k^2$), can be found at the link \url{https://github.com/andreacos/Deep-Neural-Networks-Robust-Watermark}.

We observe the following noticeable differences among the architectures we have considered:
%DenseNet and ResNet tends to have a much larger number of parameters (weights) in the convolutional layers with respect to XceptionNet \BT{Da verificare. Allora dove stanno tutti questi pesi di Xception?}. Moreover,
regardless of the primary task, the variance of the distribution of the weights of DenseNet and ResNet is much lower than the variance of the weights of XceptionNet. More precisely, the variance of non-watermarked weights for XceptionNet is in the range [0.4-2.5], while it is in
the range [4 $ \cdot 10^{-6}$- 6 $\cdot 10^{-6}$]  for DenseNet  and [2 $ \cdot 10^{-5}$- 1 $\cdot 10^{-4}$] for ResNet.
%This is the case regardless of the task accomplished by the network and is mainly a consequence of the different structure of the networks.

%
%\begin{table*}[t]
%    \scriptsize
%	\renewcommand\arraystretch{1.3}
%	\setlength{\tabcolsep}{1mm}
%	\centering
%	\caption{Summary of host layers settings and related information.}
%	\begin{tabular}{c|c|c|c|c}
%		\hline
%\multirow{2}{*}{Architecture} & {Single/} & \multirow{2}{*}{Host layers} & \multirow{2}{*}{$\sigma_k^2$} & \multirow{2}{*}{$N_k$}  \\
%& Multi-layer & & & \\ \hline
%\multirow{3}{*}{XceptionNet} & 1 layer & block14\_sepconv2 & 1.8098 & 13824\\ \cline{2-5}
% & 2 layers & block14\_sepconv1, block14\_sepconv2& [1.1378, 1.8098] & [9216, 13824]\\ \cline{2-5}
%  & 4 layers &  block13\_sepconv1, block13\_sepconv2, block14\_sepconv1, block14\_sepconv2 & [2.3441, 0.4284, 1.1378, 1.8098] & [6552, 6552, 9216, 13824]\\ \hline
%\multirow{3}{*}{DenseNet} & 1 layer & conv5\_block16\_1 & 0.00069 & 143360\\ \cline{2-5}
% & 2 layers & conv5\_block16\_1, conv5\_block20\_1& [0.00069, 0.00057] & [143360, 159744]\\ \cline{2-5}
%  & \multirow{1}{*}{4 layers} & conv4\_block16\_1, conv4\_block20\_1, conv5\_block16\_1, conv5\_block20\_1 & [0.00056, 0.00062,  0.00069, 0.00057] & [94208, 110592, 143360, 159744]\\ %\cline{3-5}
%  %&  & conv5\_block16\_1, conv5\_block20\_1, conv5\_block24\_1, conv5\_block28\_1 &  [0.00069, 0.00057, 0.00073, 0.00064] & [143360,159744, 176128, 192512] \\
%  \hline
%    \end{tabular}
%        \label{tab.info-rete-watermark}
%\end{table*}

%

In the multi-layer case, for simplicity, we watermarked the same number of weights in each layer, that is $|\Omega_k|$ is the same for all $k$'s.
Given that the host layers have different number of parameters,
the percentages of watermarked weights $p^m_k$ in different layers are different.
Such percentages obviously
%are not reported in the table since they
depend on the specific watermarking setting, i.e., on the  payload  $l$ and the spreading factor $S$.

%
%Different payloads $l$ and spreading factors $S$ were considered for the embedding.
For a given layer $k$, the watermark is embedded by considering different strengths $\gamma$ proportional to the standard variation of the distribution of the non-watermarked weights of the layer. More precisely, in the experiments, we let $\gamma = C \sigma_k/\sqrt{2}$ and adjusted the watermark strength by varying the parameter $C$.
%The percentage of watermarked weights in the embedded layers ranges from ??? ....to ????....in the case of ....
%and ....in the case of .....
We observe that, with this definition of $\gamma$, the variance of the distribution of the watermarked weights is proportional to  $\sigma_k^2$ with constant $C^2$ (in fact the variance of a $\textrm{Laplace}(\mu, \gamma)$  is equal to $2 \gamma^2$).   Therefore, $C=1$ corresponds to the case of theoretically perfect indistinguishability of the distributions.
%\footnote{We remind that $\sigma_k^2$ is the variance of the weights of the non-watermarked model in layer $k$. In the watermarked model, the variance of the non-watermarked weights in the same layer can be slightly different, hence the indistinguishability may not be perfect even for $C=1$. \MB{I would avoid this note, we already said that corresponds to {\em theoretical} indistinguishability, even because it is not even true that the coefficients follow perfectly a Laplacian distribution }}.
%
%
To watermark our models, we consider $C \in [1, 2]$.
%To watermark the GAN detection model based on XceptionNet, we considered $C \in [0.7, 1.3]$, while for the DenseNet models we considered integer values of $C$ %ranging from 2 to 20.
%
%we set  $C \in [15, 25]$.
%These values are chosen based on an analysis of the indistinguishability of the embedded watermark.
% \BT{(note that the perfect indistinguishability corresponds to $C=1$)}.
%
% on the visibility constraint,  so that the presence of  watermarked weights can not be easily detected by looking at the weight distribution.
A more detailed discussion is provided in Section \ref{sec.histograms}.

\subsection{Setting of robustness experiments }% Setting for  robustness analysis}
\label{sec.method-robustness}

With regard to model compression, we considered both parameter pruning and weights quantization.
For parameter pruning, we cropped a fraction $p$ of the
weights of the convolutional layers
% given that the watermrk is embedded in this kind of layers
by setting them
to zero.
As customary done, we cut off the weights based on their absolute values, starting from the smallest ones.
Watermark extraction is carried out as usual.
The performance are assessed for several pruning fractions $p$.
For weights quantization,  we performed conversion to integer representations by quantizing the weights with $n_b = 32, 16,8$ and $4$. % in our experiments.

As to retraining, we considered both transfer learning and fine-tuning.
For the transfer learning scenario,  we focused on the more general and challenging inductive transfer learning setting, according to which the watermarked models are re-trained on a different task and a different domain. Specifically,  we used the watermarked models as pre-trained solutions and performed the new training  using the standard cross-entropy loss $\mathcal{L}$. Of course, in this phase the watermarked weights are also updated.
We considered two transfer learning scenarios:
%  from the least to the most challenging.
%
\begin{enumerate}
\item {\em Different task with the same number of classes.} We retrained the model initially trained to solve the binary GAN  detection task to solve a new two-class classification  problem, namely the classification of image picturing cats and horses.
    %We considered \MB{Do you mean or retraining?} 20K images for each class taken from the LSUN dataset \cite{LSUN}.
    Retraining was performed on 20K images for each class, taken from the LSUN dataset \cite{LSUN}.
    % XceptionNet is used for this task.
\item {\em Different task with different number of classes}.
%We retrained the model initially trained on GTSRB data (43 classes) to solve the CIFAR-10 (10 classes) classification problem, and viceversa.
We retrained the model initially trained on CIFAR-10 data (10 classes) to solve the  GTSRB  (43 classes) and the  CIFAR-100  (100 classes) classification tasks.
\end{enumerate}
%
%\subsubsection{Different tasks with same number of classes}
% From the GAN face detection task to the  problem of image classification between cats and dogs ( \BT{@Andrea: ???? REFERENCE TO THE DATASET USED??})....... XceptionNet is used for this task.
% \subsubsection{Different tasks with different number of classes}
%From GTSRB classification (having 42 outputs)  to CIFAR-10, and to CIFAR-100. \BT{@Andrea: la input size qui e' la stessa ??} .......DenseNet is used for this task.
%%
% \subsubsection{Different tasks with (very) different input sizes and different number of classes}
%Efficient (o Dense) on Food101 to Cifar10 \BT{non so se questo alla fine lo inseriremo. Non ricordo se avevamo fatto i test alla fine. Per ora lo lascio.}

%In both cases, the networks are trained on the new task for 10 epochs,
%that can achieve good accuracy on the new task (we observed that going on with the epochs the loss does not decrease much further).

In case 1), the network was trained on the new task for 10 epochs, that are enough to achieve the maximum accuracy on the new task. For 2), transfer learning required 20 epochs on GTSRB and  200 epochs on CIFAR-100 to reach the benchmark accuracies.
%The retraining is performed by considering
%\AC{SGD optimizer, learning rate $0.01$ for 1) and $0.001$ for 2) and batch size 32}
%\BT{va riempito. Sno questi setting uguali per 1) e 2)? si puo' dire che si mantiene lo stesso lr di training?}.

We also assessed the robustness against fine-tuning, by retraining the watermarked models for 10 additional epochs on a subset consisting of 70\% of the original training data.
For the GAN detection task, where a very large number of images is available for training, we also tried fine-tuning on 30\% of the original dataset obtaining similar results.

\section{Results and discussion}
\label{sec.experiments}

In this section, we report
the results of the experiments we carried out to demonstrate
that our algorithm can achieve large payloads without impairing the
performance of the host models (Section \ref{sec.exp-perf}),  at the same time achieving outstanding robustness against network modifications and re-use (Section \ref{sec.exp-robust}).
%In particular, we show that the watermark can survive transfer learning, a result that can not be achieved by the state-of-the-art white-box methods proposed so far under the invisibility  requirement.

\subsection{Performance of DNN watermarked  models}
\label{sec.exp-perf}

We start by evaluating the drop
of classification accuracy (if any) due to the presence of the
watermark. The performance of the watermarked models are
%evaluated by measuring the test error rate of the model
%for the envisaged task, assessing the unobtrusiveness of the
%watermark, where
assessed by measuring  the
Test Error Rate (TER), defined as TER = 100\% - ACC, where ACC is the accuracy of the network on the classification task.

%***************
%Discuss payload, bit error rate and test error rate....
%(payload vs unobtrusiveness)
%***************

The results of our experiments on the GAN detection task for various
payloads ($l$), watermark strengths ($C$) and spreading factors ($S$), single and multi-layer embedding are reported in Table \ref{tab:performanceGAN}.
In the following, we find sometimes convenient to use the  compact  acronym {\em Network}-{\em Task}-$l$-$C$-$S$ to indicate the specific watermark setting.
%\MB{if we want to use this notation, we must introduce it immediately before discussing the tables }

%
% and watermark embedding parameters, namely the strength constant $C$ and the spreading factor $S$ are reported in Table ??? \BT{Tabella  simile a quella riportata nei readme}
%
\begin{table}[t]
    %\scriptsize
	%\renewcommand\arraystretch{1.3}
	\setlength{\tabcolsep}{1.5mm}
	\centering
	\caption{Performance of DNN watermarked models on the GAN detection task. The baseline  TER is 0.55 \%. In the multi-layer cases we report the percentages of watermarked coefficients for the various layers.}
    \label{tab:performanceGAN}
	\begin{tabular}{c|c|c|c|c|c}
		\hline
 {$l$} &  $S$ & {$C$} & No. L  & $p^m_k$ \% & TER \%\\ \hline
\multirow{6}{*}{256}   & 1  & 1& 1 &  1.8 & 0.6 \\ \cline{2-6}
 & 3  & 1& 1 &  5.55 &  0.4\\ \cline{2-6}
 &   12 &  1 & 1 &  22.22  &  1.3\\  \cline{2-6}
  &   12 &  1.5 & 1 &  22.22  &  1.0\\  \cline{2-6}
  &   18 &  1 & 1 & 33.33  & 0.6 \\  \cline{2-6}
 &   18 & 1  & 2 &   [25, 16.67]  & 0.6 \\ \hline
  \multirow{4}{*}{1024}   & 6 & 1 & 2 & [33.33, 22.22] & 1.0 \\ \cline{2-6}
   & 12 & 1 & 2 & [66.34, 44.44] &  1.0 \\ \cline{2-6}
      & 12 & 1.5 & 2 & [66.34, 44.44] &  1.0 \\ \cline{2-6}
     & 18 & 1 & 4 &  [70.32, 70.32, 50.00, 33.33] & 4.4  \\ \hline
    % & 18 & 1.5 & 4 & 18432 & [70.32, 70.32, 50.00, 33.33] &5.9  \\ \hline
% \multirow{2}{*}{2048}  & 3  & 1 & 2 & ?? & [??,??]  &  ?? \\ \cline{2-7}
% & 6  & 1 & 2 & ?? & [??,??]  &  ?? \\ \hline
{2048}  & 6  & 1 & 4 &  [46.88, 46.88, 33.33, 22.22]  & 1.0\\ \hline
    \end{tabular}
\end{table}

%
%The total number of watermarked  weights ($|\Omega|$) is also reported in the table. % in the table.
%
The TER of the baseline non-watermarked model is 0.55\%.
%We see that.....when .....
%
Since  the position of the host weights is known during the extraction, the integrity requirement is satisfied by construction, and the BER is always equal to 0 (not reported in the table). %
In most of the cases, a good TER can be achieved and the difference between the watermarked and the baseline models is negligible.
%, that is the same of the one achieved by the baseline non-watermarked model PUT REFERENCE TO THE GIT IN CASE.
%
%Obviously, the BER is always 0, since  the extraction always recover the embedded weights, the integrity requirement being then solved by construction. % \BT{Not even need to report a column for that}
%
%
%A good TER can be obtained with the single-layer embedding when $l = 256$.
%A good TER can be obtained with the two-layer embedding when the payload is
%$l = 1024$ bits and $S = 6$ and $12$. FINISH !!!!!!!!!!!!!!!!!!!!!!!!!!!!!!!!!
%
%
We observe that, when $S$ increases to 18 with $l = 1024$,  the number of watermarked bits starts being large $|\Omega| = 18432$, and
 embedding all the bits in the 2 layers  leads to a too large $p^m$ (with the layers almost saturated), compromising the unobtrusiveness.
%, and  a 4 layer setting has to be considered for the embedding. In this case,  TER = 0.66\%.
When the 4 layer setting is considered for the embedding in this case (see the table), we get  TER = 4.4\%, with the occupancy of the first two layers above 70\% having then a negative impact  impact on the TER.
%
%The same behavior occurs when we try to embed $l = 2048$ bits in 2 layers, with $S=6$, in which case the percentage of watermarked weights in the first layer goes far above 50\%.
%\TODO{ADJUST THIS LAST SENTENCE....insert the case also in the table}
%We found that, for $l = 2048$ bits, we need to embed the watermark in more than 4 layers to get a good TER \MB{These results are not reported in the table. Why?}.
%
%Based on our experiments, the single-layer embedding can also work with large payloads, without affecting the unobtrusiveness, when the $p^m_k$ for a given layer $k$ get close to 50\% ???? \BT{o goes above 40\%???}
%Based on our experiments, we can embed large payloads without affecting the unobtrusiveness, as long as the $p^m_k$ for a given layer $k$ remains below 40\%.
%We observed that,
%Therefore,
%%as a general behavior,
%In general, from our experiments we observed that,
As a general observation,
in order to embed large payloads without affecting the unobtrusiveness constraint, it is often necessary to consider  multiple layers, in such a way that the percentage of watermarked weights in the embedded layers does not grow not too much.

\begin{table}[t]
    %\scriptsize
	%\renewcommand\arraystretch{1.3}
	\setlength{\tabcolsep}{1.5mm}
	\centering
	\caption{Performance of   ResNet-based CIFAR-10 watermarked model. The baseline TER is 5,1\%. In the multi-layer cases the percentages of watermarked coefficients for the various layers is reported.}
    \label{tab:performanceCIFARnew}
	\begin{tabular}{c|c|c|c|c|c}
		\hline
{$l$} &  $S$ & {$C$} &No.L  & $p^m_k$  \%& TER  \%\\ \hline
\multirow{4}{*}{256}
&  3 & 1  & 2  & [0.04, 0.01]  &  5.3 \\ \cline{2-6}
 & 25  & 1 &  2  &   [0.54, 0.14] &   5.1\\ \cline{2-6}
  & 25  & 1.5 &  2  &   [0.54, 0.14] &   5.1\\ \cline{2-6}
 & 50  & 1 &  2  & [1.09, 0.28]  &  5.2 \\ \hline
\multirow{5}{*}{1024}
&  25 & 1 & 2  &  [2.17, 0.54] & 5.0 \\ \cline{2-6}
&  25 & 1.5 & 2  &  [2.17, 0.54] & 5.1 \\ \cline{2-6}
&  50 & 1 & 2  &  [4.34, 1.09] & 5.2 \\ \cline{2-6}
%&  6 & 1 &1  & 6144 & [2.14, 1.90]&  5.33\\ \cline{2-7}
&  75 & 1 & 2  & [6.51, 1.63] &  6.5 \\ \cline{2-6}
& 100  & 1 & 2 & [8.68, 2.17] &  5.1\\ \hline
\multirow{2}{*}{2048}  & 75 & 1 & 2 &  [13.02, 3.26] &  5.1\\ \cline{2-6}
& 100 & 1& 2 &  [17.36, 4.34] &  5.0 \\ \hline
\multirow{3}{*}{4096}  &  125 & 1  & 4 & [21.70, 21.70, 5.53, 5.53] &  5.3\\ \cline{2-6}
& 125 & 1.5& 4 &  [21.70, 21.70, 5.53, 5.53] &  5.2\\  \cline{2-6}
& 150 & 1 & 4 & [26.04, 26.04, 6.51, 6.51] & 5.3 \\ \hline
\multirow{2}{*}{8192}  &  150 & 1  & 4 &
[52.08, 52.08, 13.02, 13.02]  &  5.2\\ \cline{2-6}
%& 150 & 1.5& 4 & [52.08, 52.08, 13.02, 13.02] &  5.4\\  \cline{2-6}
& 200 & 1 & 4 & [69.44, 69.44, 17.36, 17.36] & 5.5 \\ \hline
\multirow{4}{*}{16384}  &  \multirow{2}{*}{250} & \multirow{2}{*}{1}  & \multirow{2}{*}{8} & [34.72, 17.36, 34.72, 34.72,
 &  \multirow{2}{*}{5.2}\\ % \cline{2-6}
 & & & &   34.72, 43.4, 43.4, 43.4] & \\ \cline{2-6}
%& 400 & 1 & 8 & [55.56, 27.78, 55.56, 55.56, 55.56, 69.44, 69.44, 69.44] & 4.8 \\ \hline
&  \multirow{2}{*}{400} & \multirow{2}{*}{1}  & \multirow{2}{*}{8} & [55.56, 27.78, 55.56, 55.56,
 &  \multirow{2}{*}{ 4.8}\\ % \cline{2-6}
 & & & &  55.56, 69.44, 69.44, 69.44]  & \\  \hline
    \end{tabular}
\end{table}

Table \ref{tab:performanceCIFARnew} reports the results for CIFAR-10 when a ResNet18 architecture is used for the classification.
Again, the BER is not reported, being always equal to 0 by construction.
The TER
of the baseline non-watermarked model is 5.1\%.
%, which is similar to the TER achieved by the watermarked models.
%Noticeably, in many cases, the TER of the watermarked models is  lower, and the watermark acts as a regularizer.\BTcomm{E' ancora vero?}
%
A similar behavior is obtained using DenseNet (the TER results are summarized in Table \ref{tab:trasferLearningCIFAR10DenseNet}, column 5. The baseline TER of the non-watermarked model is equal to 4.7\%) .
Given that  ResNet and DenseNet have  significantly larger number of parameters  than XceptionNet in the convolutional layers  (this is especially the case with ResNet), many more weights can be used to carry the  watermark before reaching a critical value of occupancy.

\subsubsection{Analysis of weights distribution}
\label{sec.histograms}

In this section, we analyze the distributions of the networks weights for the non-watermarked and watermarked models to experimentally validate the assumption, made in the theoretical analysis, that the distribution of the non-watermarked weights is similar for watermarked and non watermarked models.
%A qualitative analysis is performed by visualizing the histograms of the weights for the entire network and the watermarked layer only.
The analysis is carried out by considering 3 different watermark settings. Similar  results are achieved in the other settings.

Fig. \ref{fig.Laplacian_distrib}  shows the  distribution of the weights of a non-watermarked model (left) and a the non-watermarked weights of a watermarked model (right) for the case of Xception-based GAN detection (first row) and  CIFAR-10  classification based on ResNet (second row) and DenseNet (third row) for one of the watermarked settings we have considered.
%The watermark setting for the GAN detector is $l = 256$, $S = 18$, $C = 1$, 2-layer embedding, while for the  GTSRB classifier we have $l = 1024$, $S = 25$, $C= 1$, 4-layers.
%
A similar behavior is observed in the other settings.
By looking at the distributions of the weights of non-watermarked models, we  observe that they approximate reasonably well a Laplacian distribution  (the Laplacian fit is reported in the plots).
Moreover, the presence of the watermark does not significantly affect the overall distribution, the shape being similar for both the non-watermarked and watermarked models.

Fig. \ref{fig.indistinguish}  shows the distribution of the watermarked (red) and non-watermarked (blue) weights in the embedding layer for the watermarked models. From top left to bottom right,
the plots refer to the GAN watermarked model,  ResNet-based and DenseNet-based CIFAR-10  watermarked model for the same setting as in Fig. \ref{fig.Laplacian_distrib}.
The name of the embedding layer visualized in the plots is reported in the figure.
%\footnote{In all these settings,  robustness against re-use is achieved (see Section \ref{sec.exp-robust}).}.

%
For the  GAN detection model, the plot shows the distribution of the weights of the single watermarked layer 'block14$\_$sepconv2'\footnote{We remind that the position of the watermarked weights - corresponding to the red distribution in the plots - is secret.}.
% \BTcomm{dovrebbe essere ovvio...ma per uno dei revisori TIFS non lo era. Per cui ho aggiunto nota}}}.
Since the sample variance of the distributions of the watermarked and non-watermarked weights is very similar  with $C=1$ (the former being $\sigma_k$ = 1.3252 and the latter $\sigma_k$ = 1.3205 for the setting in the figure)
the setting is good from  the point of view of the security.
%%no visible tails are introduced, thus confirming that \sout{ the values $C\le 1$ are} \BT{the setting is} good from the point of view of the security. \MB{Do we say anywhere that for $C= 1$ we should get perfect indistinguishability? I think we should mention this fact when we first introduce the parameter $C$.}
%When $C$ is larger than $1$, instead, we observed the tails start becoming visible. \BTcomm{In view of this I have removed the results for the case $C = 1.3$ in the next table for the robustness results.}
Note that this is a consequence of the goodness of the approximation made in the theoretical analysis, that the distribution (and then the variance) of the non-watermarked weights in any given layer remains similar for watermarked and non watermarked models.
For the ResNet and DenseNet CIFAR-10 watermarked models, the histograms of the weights of layer
'layer4.0.convbn$\_$2' and 'dense4.29.conv1'  are visualized respectively, for which the watermark occupancy (percentage of watermarked weights) is 2.17\%
%8.68\%
and 19.13\%.
Expectidely, for the settings and/or embedding layers for which the percentage of the watermarked weights is very small, it is hard to visualize the distribution of the watermarked weights in the plots.
The variances of the watermarked and non-watermarked weights are respectively $4.3 \times 10^{-4}$ and  $4.9 \times 10^{-4}$  for 'layer4.0.convbn$\_$2' and $2.1 \times 10^{-4}$ and  $2.3 \times 10^{-4}$ for  'dense4.29.conv1'.
We also measured the KL distances between the distributions of the watermarked and non-watermarked weights  for the 3 settings and embedding layers reported in Figure \ref{fig.indistinguish}, that are 0.066,  0.678, and  0.082 respectively.

\begin{figure}[]
	\centering
\subfigure[XceptionNet-based GAN detection.]% Non-watermarked model (left). Watermarked model (right).]
{
	\includegraphics[scale=.265]{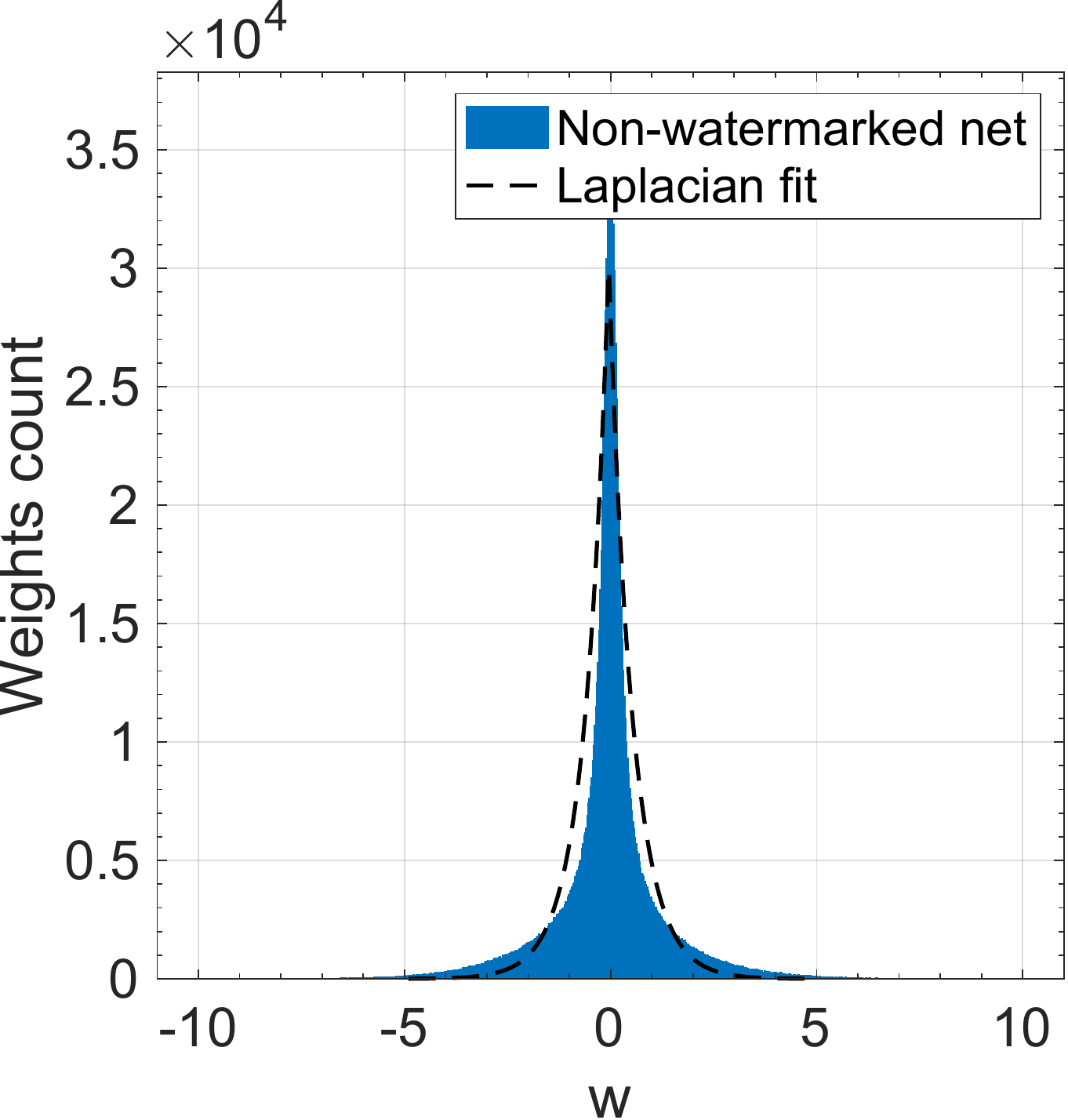}
	\includegraphics[scale=.265]{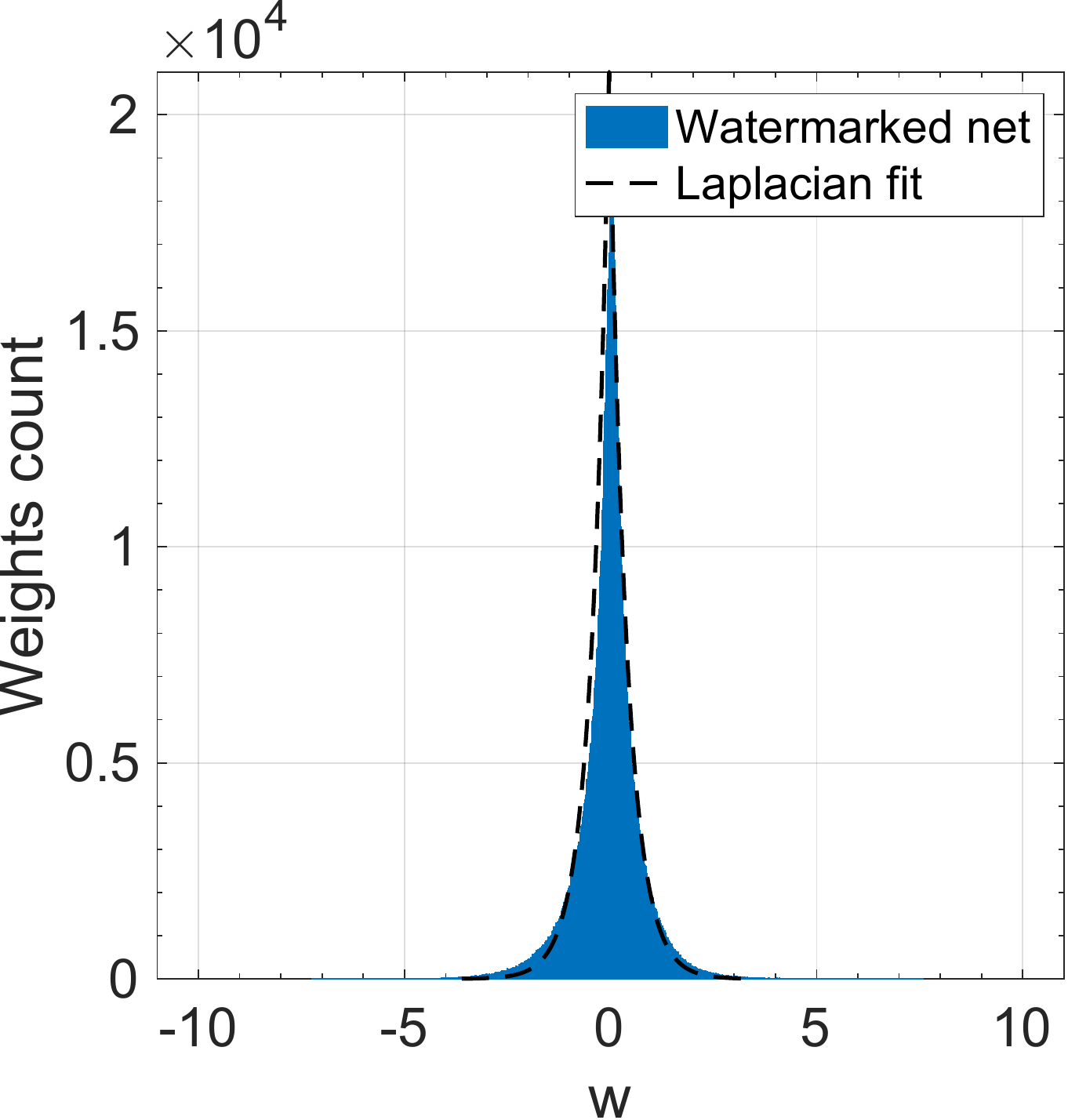}}
\subfigure[ResNet-based CIFAR-10 classification.] % Non-watermarked model (left). Watermarked model (right).]{
{
	\includegraphics[width=0.41\columnwidth]{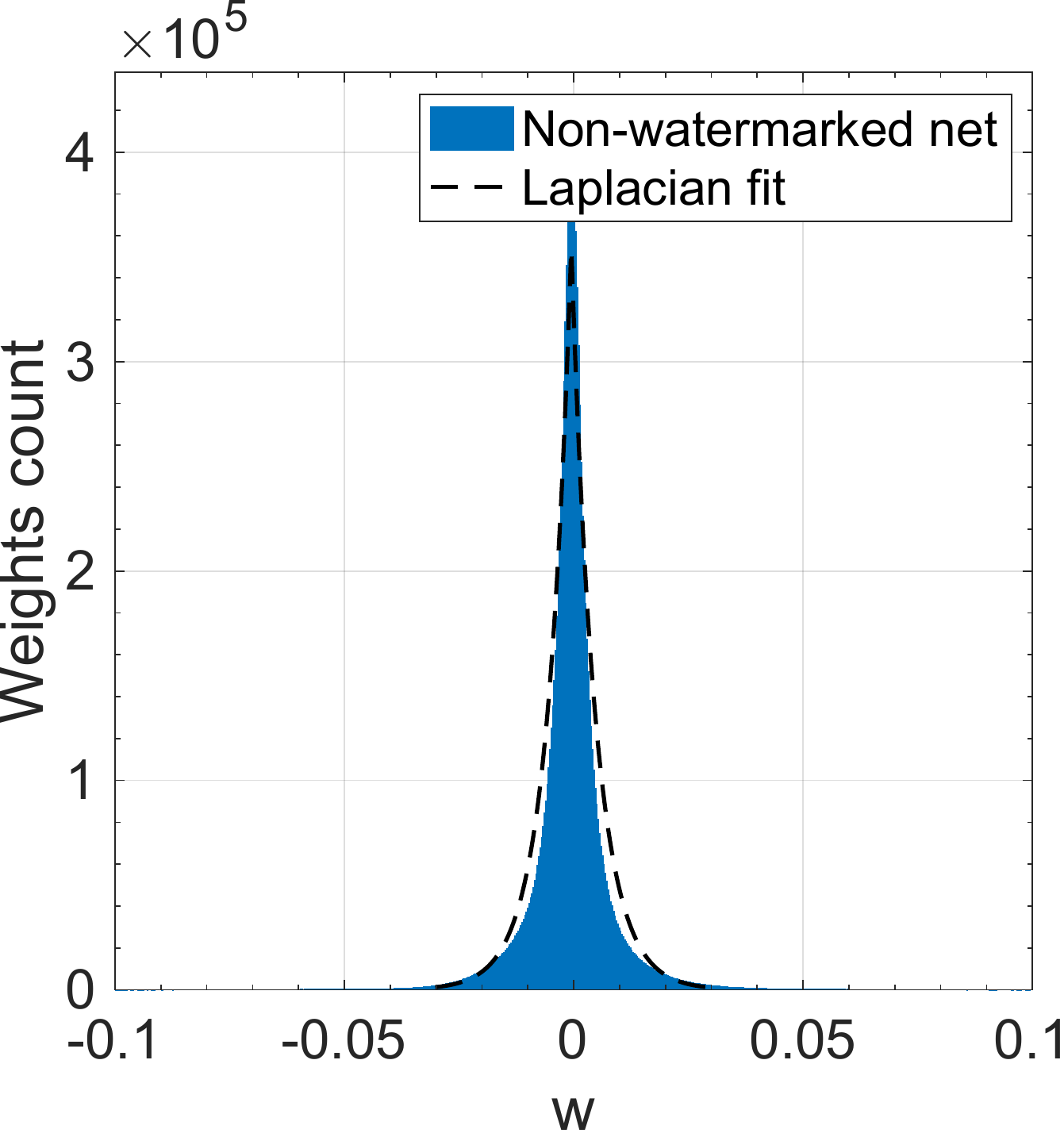}
	\includegraphics[width=0.41\columnwidth]{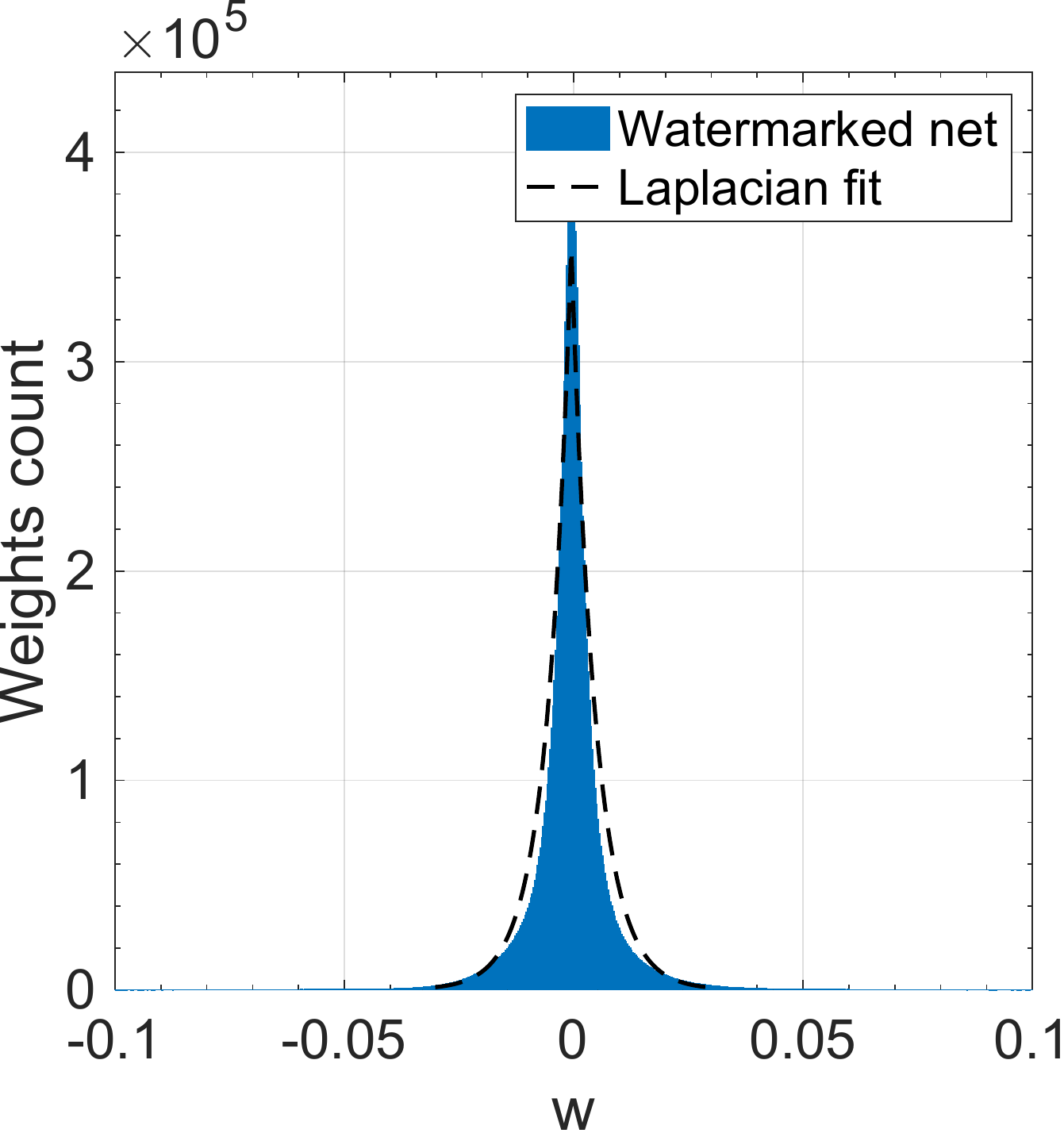}}
	%\includegraphics[scale=.33]{XCeptionNet-Watermarked-K-1-b-256-S-6-L-1-block14_sepconv2-Weights-Watermarked-Only.pdf}
    %\includegraphics[scale=.33]{XCeptionNet-Watermarked-K-1-b-256-S-6-L-1-block14_sepconv2-Weights-Separated.pdf}
%	\subfigure[]{
%	\includegraphics[width=0.48\columnwidth]{Densenet-Non-Watermarked-All-Net-Weights-Weights-With-Laplacian-Fit.pdf}}
%	\subfigure[]{
%	\includegraphics[width=0.48\columnwidth]{Densenet-Watermarked-K-15-b-1024-S-15-L-4-All-Net-Weights-With-Laplacian-Fit.pdf}}\\
%\subfigure[]{
%	\includegraphics[c]{Densenet-Watermarked-K-15-b-1024-S-15-L-4-conv4_block20-Weights-With-Laplacian-Fit.pdf}}
%    \subfigure[]{
%    \includegraphics[d]{Densenet-Watermarked-K-15-b-1024-S-15-L-4-conv4_block20-Weights-With-Laplacian-Fit.pdf}}
\subfigure[DenseNet-based CIFAR-10 classification.] % Non-watermarked model (left). Watermarked model (right).]{
{
	\includegraphics[width=0.42\columnwidth]{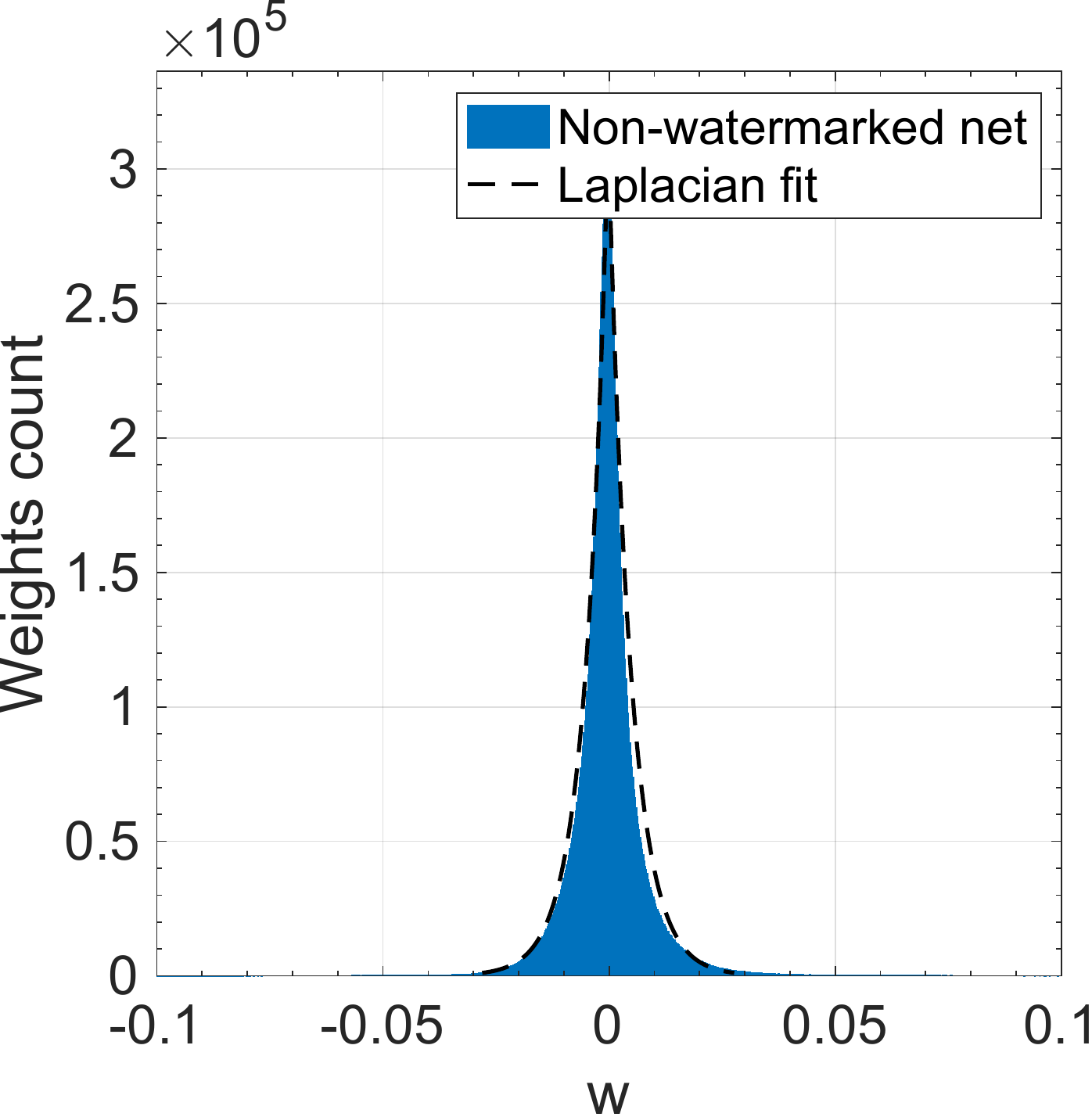}
	\includegraphics[width=0.42\columnwidth]{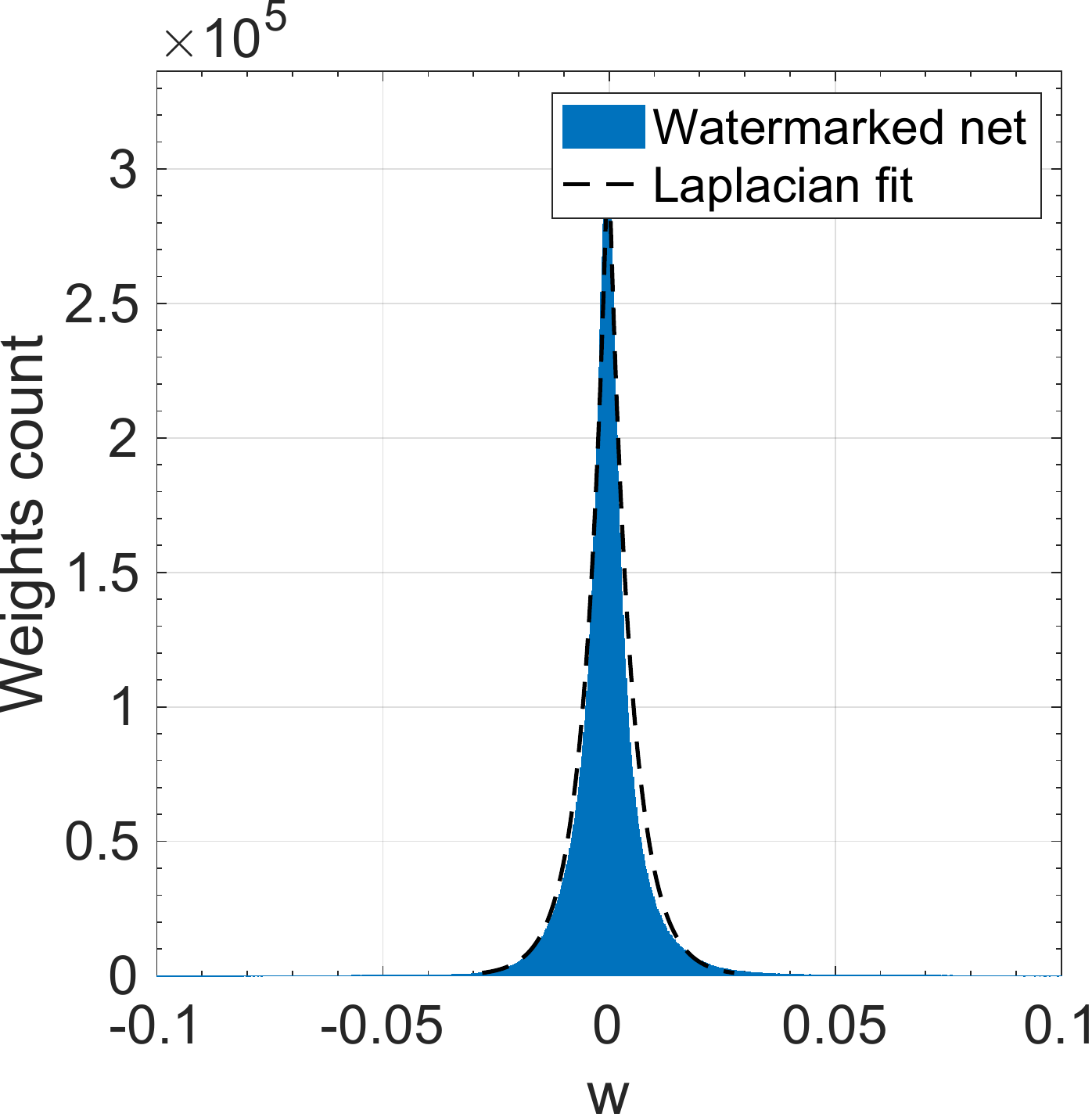}}
	\caption{
Distribution of  non-watermarked weights for non-watermarked (left) and watermarked  (right) models, for
 XceptionNet-based GAN detection (a) ResNet-based CIFAR-10 classification (b) and DenseNet-based CIFAR-10 classification (c). The watermark settings are XceptionNet-GAN-256-1-18,  ResNet-CIFAR10-1024-1-100 and DenseNet-CIFAR10-1024-1-75  respectively. % \MB{Well, the fitting is not that perfect}
  %\TODO{fill and new PLOTS}} \AC{font non in accordo con i plot di Xception (prima riga). La legenda e' troppo piccola. Va tolto il titolo dal plot.}}
  }
 %The watermark settings in the right plots are: $l = 256$, $S = 18$, $C = 1$, 2-layers, for the GAN detector and  $l = 1024$, $S = 25$, $C = 1$, and 4-layers, for the GTSRB classifier.
 %
 %
  %\MB{Insert a subcaption or a legend in the plots to indicate what they are referring to.}
%\MB{ Nella prima figura il fit non e' poi un granche'. Mi piacerebbe capire perche' fissando alcuni coefficienti i coefficienti non marchiati diventano piu' laplaciani }
	\label{fig.Laplacian_distrib}
%\vspace{-1cm}
\end{figure}

%\enlargethispage{\baselineskip}

%\begin{figure}[]
%	\centering
%	%\includegraphics[scale=.33]{XCeptionNet-Watermarked-K-1-b-256-S-6-L-1-block14_sepconv2-Weights-Watermarked-Only.pdf}
%    \includegraphics[scale=.32]{XCeptionNet-Watermarked-B-256-C-1-S-18-L-2-block14_sepconv2-Weights-Separated.pdf}
%    %\includegraphics[scale=.32]{Densenet-Watermarked-K-15-b-1024-S-15-L-4-conv4_block20-Weights-Watermarked-Only.pdf}
%    \includegraphics[scale=.32]{Densenet-GTSRB-Watermarked-B-1024-C-1-S-25-L-4-conv4_block20-Weights-Separated.pdf}
%    %	\includegraphics[scale=.32]{Densenet-Watermarked-K-15-b-2048-S-25-L-2-conv5_block16-Weights-Watermarked-Only.pdf}
%    \includegraphics[scale=.32]{Densenet-GTSRB-Watermarked-B-2048-C-1-S-25-L-4-conv5_block16-Weights-Separated.pdf}
%    %
%	\caption{Distribution of  the weights  in the embedding layer. From top left to bottom right: XceptionNet-based GAN detection model watermarked with $l = 256$, $S = 18$, $C = 1$, 2-layers (block14$\_$sepconv2 is visualized); DenseNet-based GTSRB classification  model watermarked with  $l = 1024$, $S = 25$, $C = 1$, and 4-layers (conv4$\_$block20 is visualized); DenseNet-based GTSRB classification  model watermarked with  $l = 2048$, $S = 25$, $C= 1$, 4-layers (conv5$\_$block16 is visualized). }
%	\label{fig.indistinguish}
%\end{figure}

\begin{figure}[]
	\centering
    \includegraphics[scale=.32]{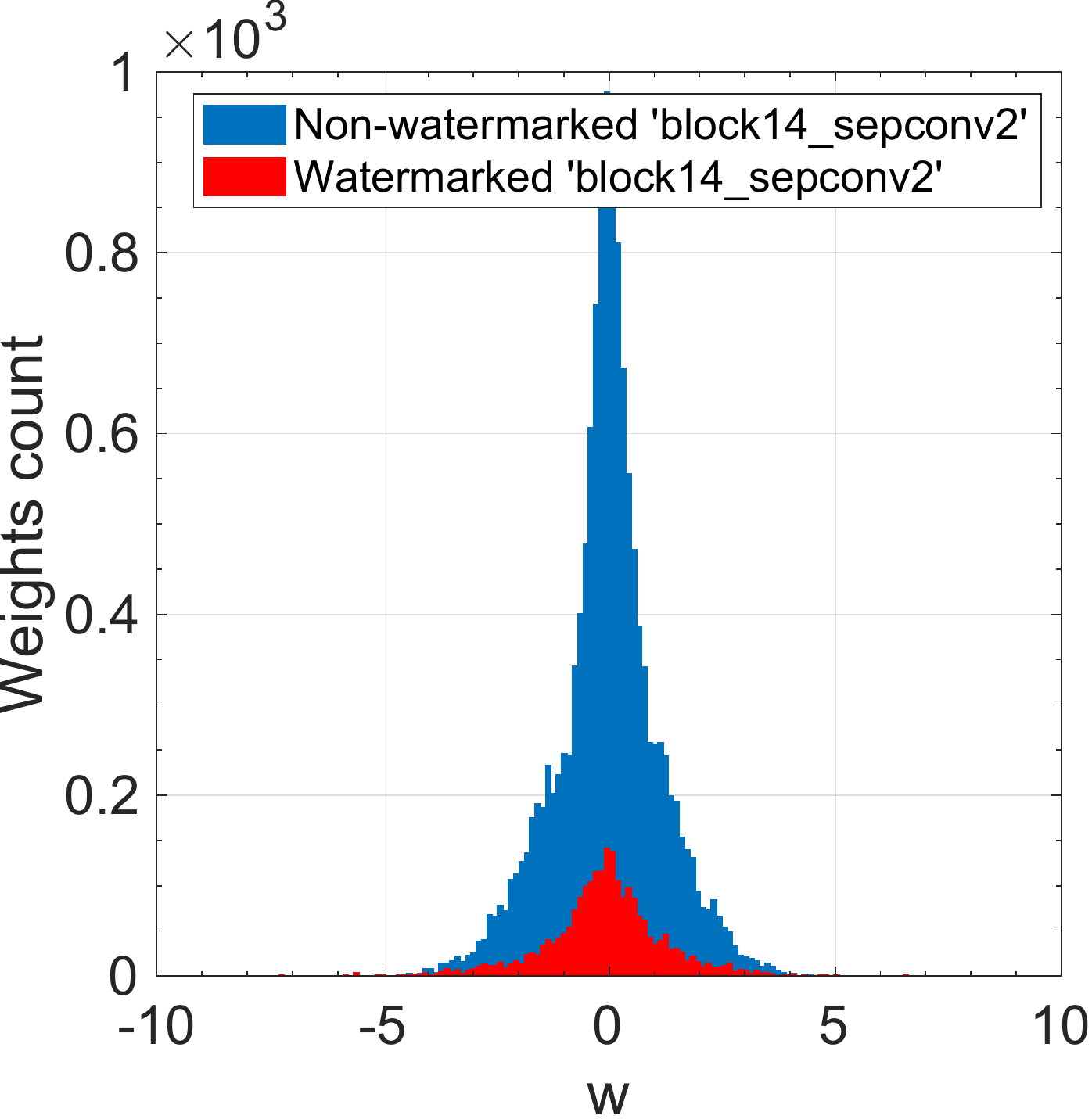}
    \includegraphics[width=0.48\columnwidth]{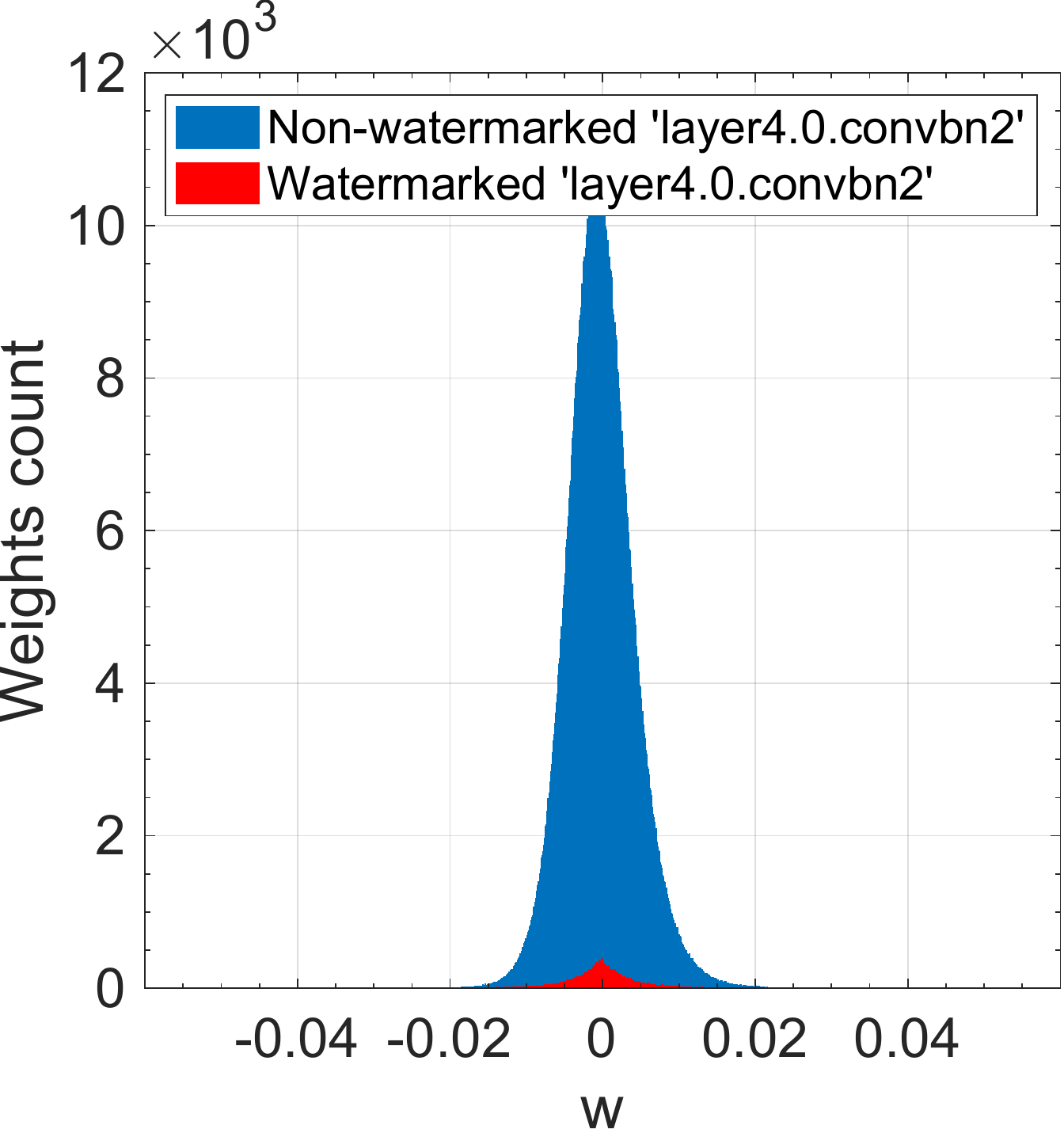}
    \includegraphics[width=0.47\columnwidth]{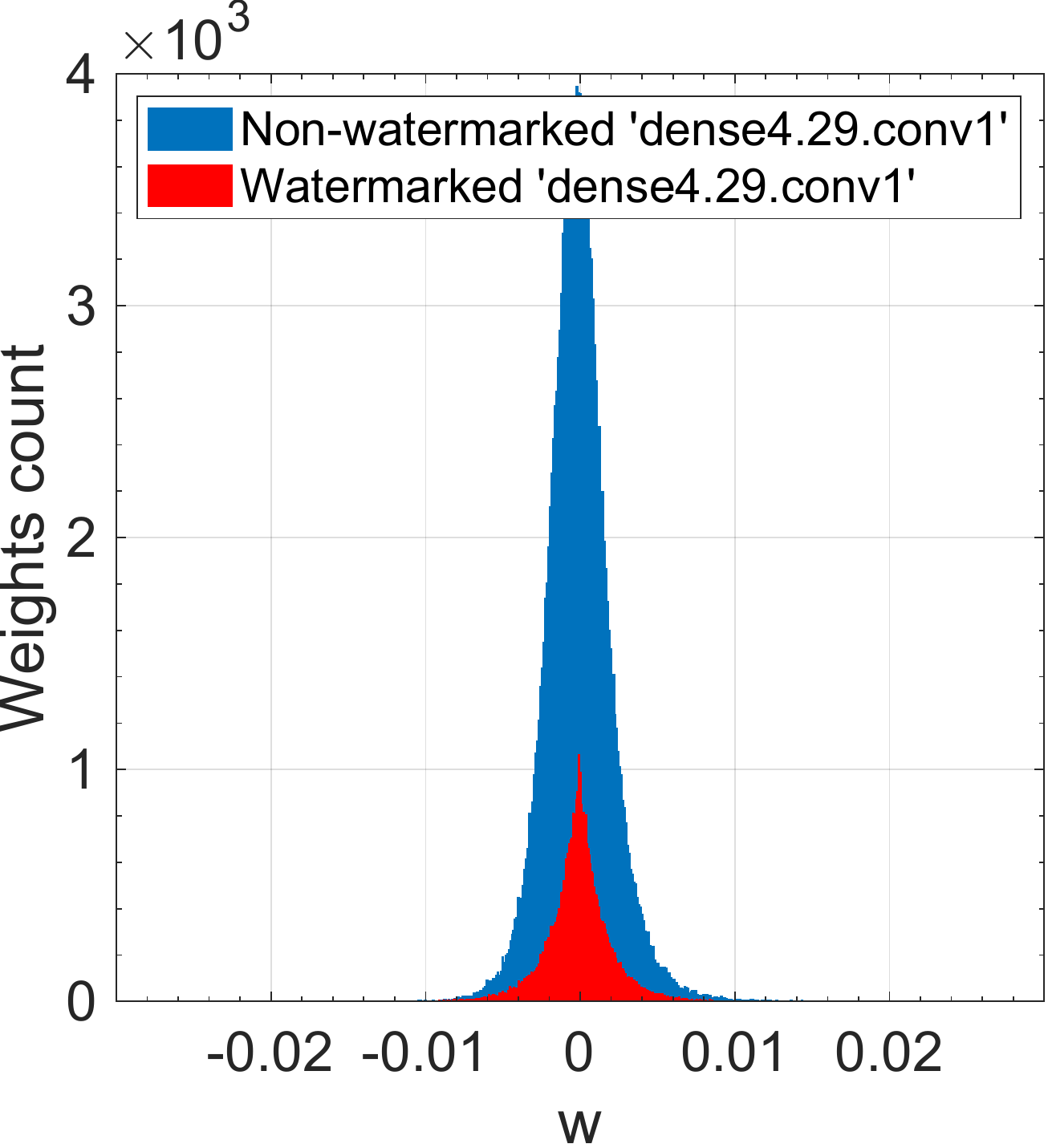}
	\caption{Distribution of  the weights  in the embedding layer. From top left to bottom right:
%XceptionNet-based GAN detection model watermarked with $l = 256$, $S = 18$, $C = 1$, 2-layers (block14$\_$sepconv2 is visualized); DenseNet-based GTSRB classification  model watermarked with  $l = 1024$, $S = 25$, $C = 1$, and 4-layers (conv4$\_$block20 is visualized); DenseNet-based GTSRB classification  model watermarked with  $l = 2048$, $S = 25$, $C= 1$, 4-layers (conv5$\_$block16 is visualized).
XceptionNet-GAN-256-1-18 (block14$\_$sepconv2 is visualized); ResNet-CIFAR10-1024-1-100 (layer4.0.convbn$\_$2  is visualized); DenseNet-CIFAR10-1024-1-75 (dense4.29.conv1 is visualized).}
	\label{fig.indistinguish}
\end{figure}

\subsection{Robustness evaluation}
\label{sec.exp-robust}

In this section, we report the results assessing the robustness of the proposed watermarking
algorithm against parameters pruning, weights quantization, transfer learning and fine tuning.
The analysis of the robustness against retraining  -  in particular transfer learning - is the most interesting one, being this the most challenging requirement,
hence  most of the results we are reporting refers to this case.

In synthesis, the results we got show that, when  the spreading factor $S$ is large enough,  the watermark embedded with $C = 1$ (perfect theoretical indistinguishability) is extremely robust against all kinds of network modification and re-use.
%thus confirming the good performance of the proposed method based on spread-spectrum coding.
%% without affecting the  robustness.

\subsubsection{Model compression}

%The easiest way to compress a DNN model, is parameters pruning, that is cropping a fraction $p$ of the
%weights of the convolutional layers
%% given that the watermrk is embedded in this kind of layers
%by setting them to zero.
%As customary done, we cut off the weights based on their absolute values, starting from the smallest ones.
%
%Watermark extraction is carried out as usual.
%The performance are assessed for several pruning fractions $p$.
%
Fig. \ref{fig:pruning}
%\BT{@Andrea: plot. Nella legenda delle due curve ci va scritto Acc(\%) e BER(\%). Nell'asse x ci va scritto $p$. Usa un tratto diverso per le due linee (esempio dotted o continua per la Acc e tratteggiata per il BER) oltre che un colore diverso se vuoi.}
reports the results of robustness against parameter pruning that we got for the XceptionNet-based GAN detector (top) and the ResNet-based and DenseNet-based CIFAR-10 classifiers (bottom) in the same settings considered above. A very similar behaviour can be observed in the other settings. For each model, the TER and BER for various values of $p$ are reported.
%For sake of visualization, the Acc (=1-TER) is reported instead of the TER.
%
%
%{\em The parameter setting considered  for the  watermarked GAN detection network is $l= 256$, $S=18$, $C = 1$, 2-layers. The GTSRB classifier is watermarked with $l = 1024$, $S = 25$, $C = 1$, 4-layers, and the CIFAR-10 classifier with $l = 1024$, $S = 25$, $C = 1$, 2-layers.
% in the left plot and with $l = 2048$, $S = 25$, $C = 15$ 2-layers  in the right plot.
%For the case of CIFAR-10, the TER-Top3 values are reported (the behavior of the TER-Top1 being similar as $p$ increases)\footnote{TER-Top1 is conventional TER, while TER-Top3 refers to the TER when we measure the probability that the correct class is included among the 3 classes with the highest predicted values.}.}
%}
We see that,  pruning has almost no impact on the TER  when $p$ is lower than $0.4$/$0.6$. The BER remains zero until $p= 0.6$ in all the cases and starts increasing  when the network is no longer usable (TER $\ge 50\%$) and in any case the unobtrusiveness requirements is compromised, thus confirming the robustness of the watermark against model pruning.
%\TODO{Plot per DenseNet con S = 100}.

Robustness against weights quantization is also achieved. In particular, we verified that
% conversion to int32, int16, int8 and int4  does not affect the BER in all the settings considered for the various tasks,
in all the settings considered for the various tasks, even in the case $n_b = 4$, conversion to integers does not affect the BER, while the TER increases to a value above or around 50\% with $n_b = 8$, for the XceptionNet-based GAN detectors, and with $n_b = 4$, for the ResNet-based and DenseNet-based CIFAR-10 classifiers.
%that is when $n_b = 8$ or $4$ bits are considered for the quantization.
Robustness against quantization is a consequence of the fact that the sign of the weights is preserved by the quantization operation, hence the extraction of the watermark is not affected by the quantization operation.

%The influence of pruning on the TER in the various cases is shown in Table \ref{tab:pruning}\BT{@Andrea: simile alla Tabella 13 del paper di Yue. COn le seguenti colonne 'Network e Task', 'Embedded layer', 'payload', 'p\%', TER, e infine BER. Se i risultati dove si va a prunare i modelli dopo il trasfer learning sono buoni si potrebbero includere anche quelli inserendo una colonna che specifica se il modello è originale o TL (oppure un'altra tabella)},
%where we report the results in the case of ??????.......
%with embedding layer ??????? and  payload $l = ????$
%bits, for different values of the pruning percentage. As shown
%in the table, when the pruning percentage is equal to ??\%,
%the TER of the two models is already very large (equal to ???\%), however, we got BER=0. Moreover, we also see that our watermarking algorithm
%can resist even larger pruning percentages.

%
%
%\begin{table}[t]
%    %\scriptsize
%	%\renewcommand\arraystretch{1.3}
%	%\setlength{\tabcolsep}{1mm}
%	\centering
%	\caption{Robustness performance against parameter pruning. \BT{TO FILL}}
%    \label{tab:pruning}
%	\begin{tabular}{c|c|c|c|c}
%		\hline
%& & & & \\
%		\hline
%    \end{tabular}
%\end{table}

\begin{figure}[]
	\centering
     \includegraphics[scale=.23]{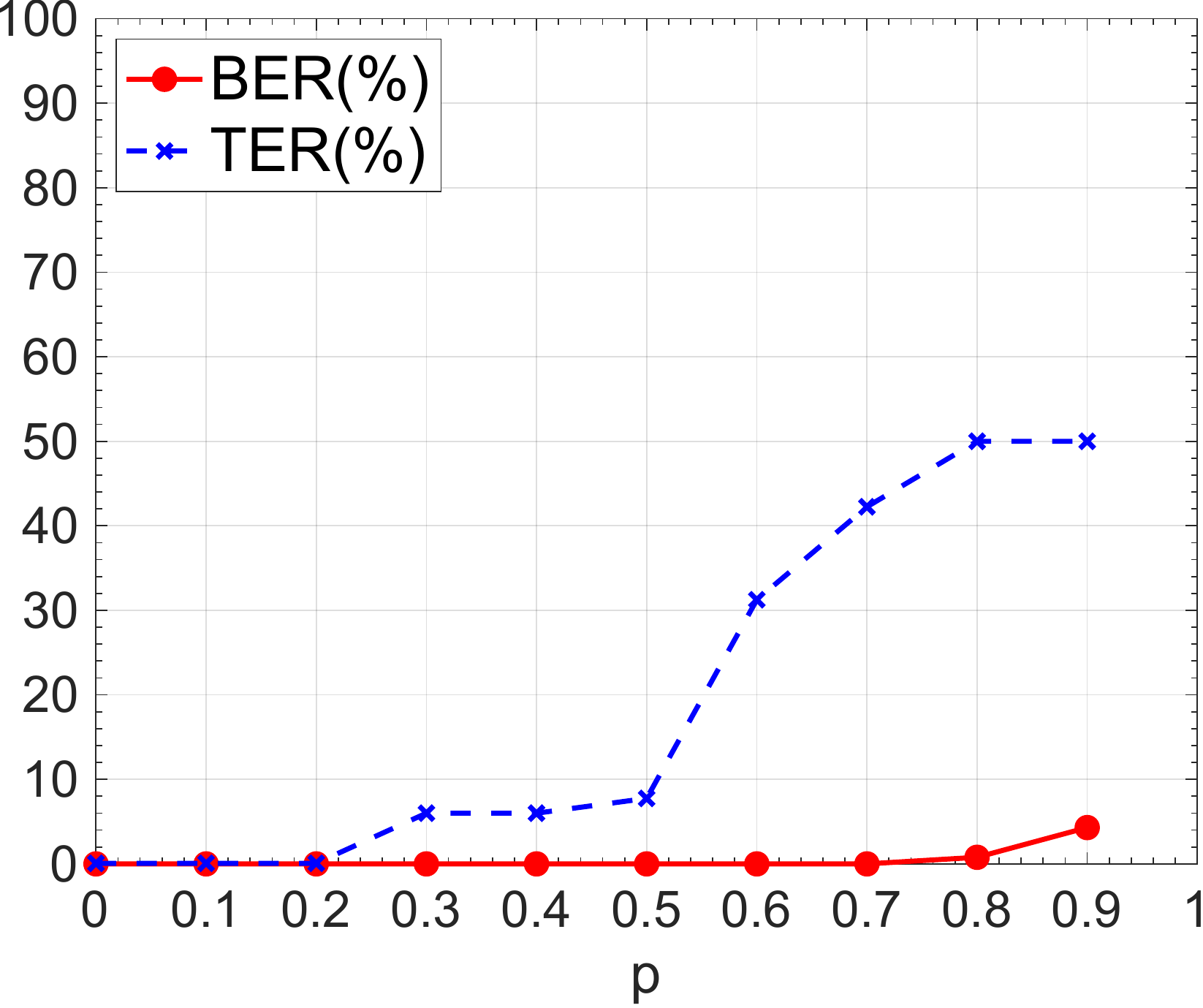}
          \includegraphics[width = 0.435\columnwidth]{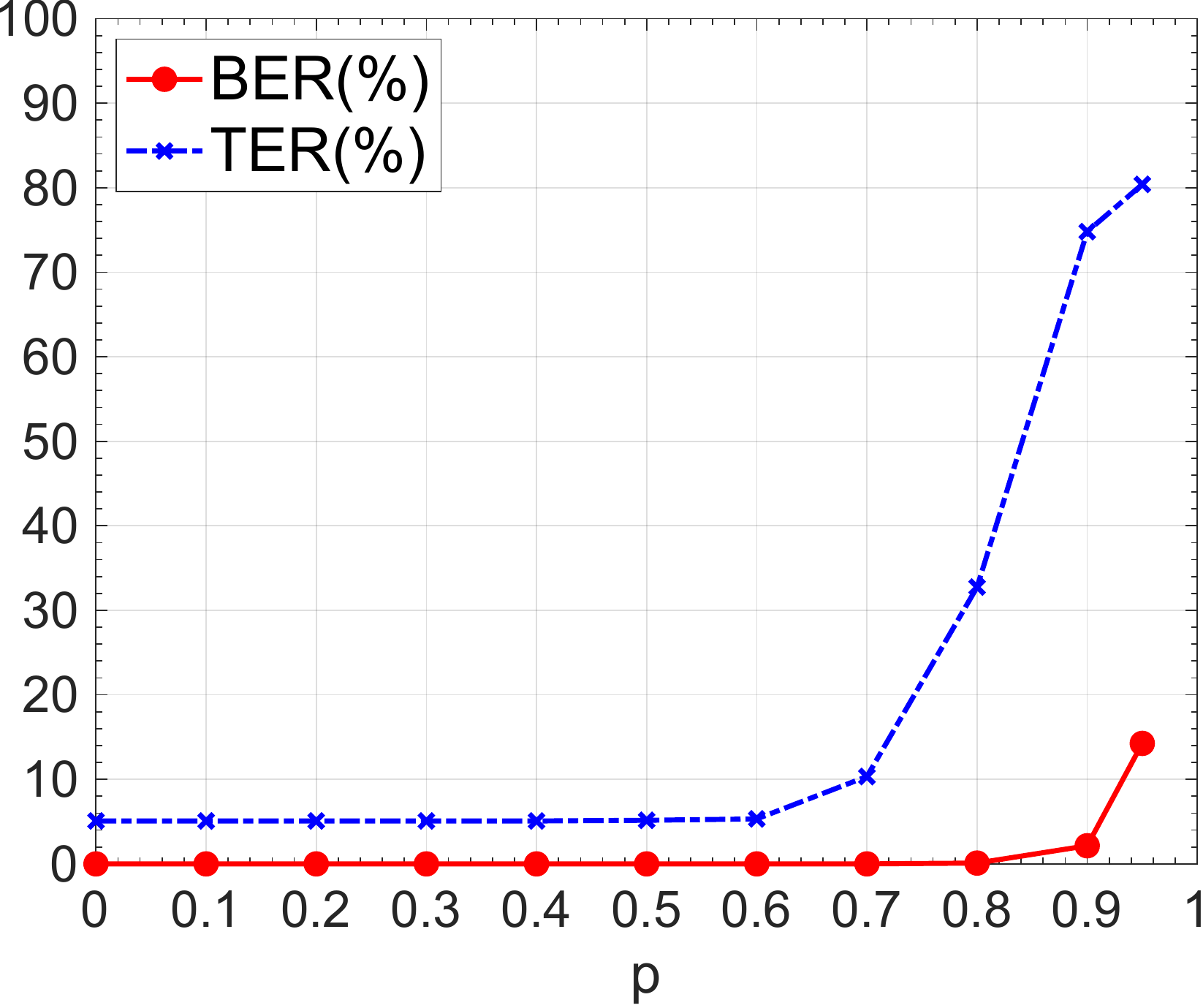}\\
                    \includegraphics[width = 0.435\columnwidth]{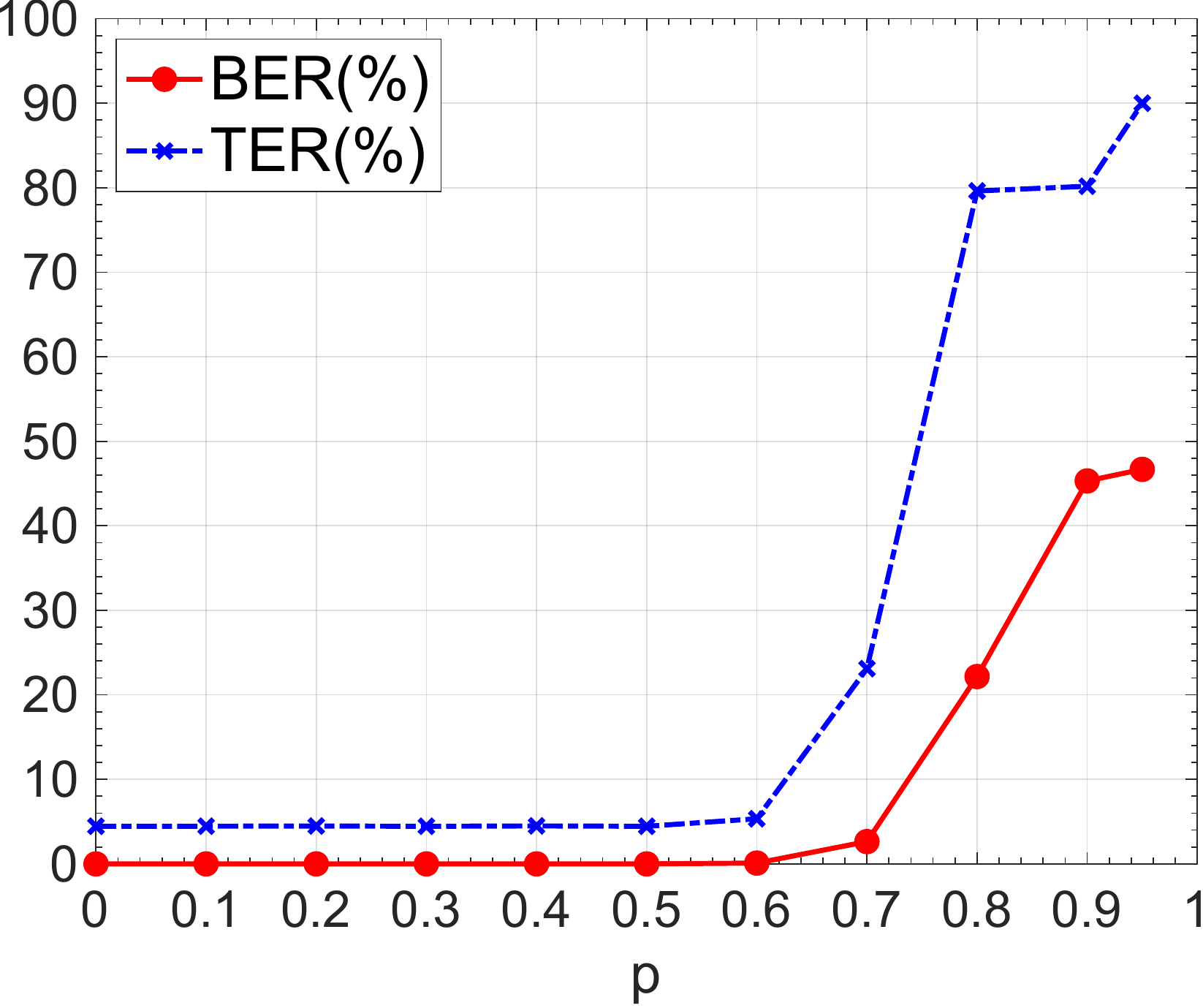}
	\caption{Robustness against parameter pruning as a function of the pruning fraction $p$. From top left to bottom right: XceptionNet-GAN-256-1-18; ResNet-CIFAR10-1024-1-100; DenseNet-CIFAR10-1024-1-75.
% \TODO{new Plots} \AC{Manca uniformita'. Nel plot per Dense lo screenshot riporta in grigio le funzionalita' in alto}
%watermarked GAN detector  with $l= 256$, $S=18$, $C = 1$, 2-layers embedding;   watermarked  GTSRB classifier with $l = 1024$, $S = 25$, $C = 1$, 4-layers  embedding; watermarked  CIFAR-10 classifier  with  $l = 1024$, $S = 25$, $C = 1$, 2-layers embedding.
}
	\label{fig:pruning}
\end{figure}

\subsubsection{Retraining}%{Robustness against transfer learning}

%{\em To measure the robustness of the watermarked models
%against transfer-learning, we used them as pre-trained solution and performed a new training on a different task and dataset
%using the standard loss, i.e, the cross-entropy loss $\mathcal{L}$, without frozing any parameters (the watermarked weights being updated as well).}

Table \ref{tab:trasferLearningXception} shows the results wegot by retraining the XceptionNet
 GAN detection  watermarked  models (see Section \ref{sec.method-robustness} for the setting of these experiments). %As detailed in Section \ref{sec.method-robustness}, fine-tuning is performed on a subset of 70\% of the original dataset, while retraining for transfer learning is performed considering the  cat\&horse LSUN classification task.
The results are reported for the payloads, strengths and spreading factors.
%The Acc is given by  100 - TER(\%).
%\BT{@Andrea: quale e' l'accuratezza del task catti vs cavalli del modello standard addestrato da scratch (non fine tuned)?}
%
\begin{table}[t]
    %\scriptsize
	%\renewcommand\arraystretch{1.3}
	\setlength{\tabcolsep}{1.2mm}
	\centering
	\caption{Robustness  of  watermarked models for GAN  detection, against fine tuning (on 70\% of the dataset) and transfer learning  to cat\&horse LSUN classification ( baseline TER  3.3\%).
%Both TER and BER (\%) refer to the cat\&horse classifier after transfer learning.
%The baseline TER for  cat\&horse classification is 3.3\%.
L = Layers.  FT = Fine Tuning.  TL = Transfer Learning. The BER values below 0.1\% are highlighted in bold (the number of bit errors is report among brackets). }
    \label{tab:trasferLearningXception}
	\begin{tabular}{c|c|c|c||c|c| c| c}
		\hline
\multirow{2}{*}{$l$} &  \multirow{2}{*}{$S$} & \multirow{2}{*}{$C$} &   \multirow{2}{*}{No.L}  & \multicolumn{2}{c|}{After FT}    & \multicolumn{2}{c}{After TL}   \\ \cline{5-8}
& & & & TER  &  BER    & TER &  BER    \\ \hline
\multirow{6}{*}{256}
 & 1  & 1& 1 & 0.1  &  6.25 (16)      &   7.8  & 16.79   (43)  \\ \cline{2-8}
 & 3  & 1& 1 & 0.5 & {\bf 0} & 7.5 & 4.69 (12)   \\ \cline{2-8}
& 12  & 1& 1 & 1.4& {\bf 0}  &  5.1 & 1.17 (3) \\ \cline{2-8}
& 12  & 1.5& 1 & 0.6  & {\bf 0}  & 8.3 & {\bf 0}    \\ \cline{2-8}
& 18  & 1& 1 & 0.1 & {\bf 0}   & 7.5 &  0.78 (2)  \\ \cline{2-8}
 & 18  & 1  & 2 & 0.1 & {\bf 0}  & 7.5& {\bf 0}   \\  \hline
 \multirow{4}{*}{1024}
  & 6 & 1 & 2 & 0.6  & {\bf 0}    & 7.7 & 2.34 (24) \\ \cline{2-8}
  & 12 & 1 & 2  & 0.5 &  {\bf 0}  & 8.3 &   0.78 (8)  \\ \cline{2-8}
   & 12 & 1.5 & 2  & 0.4 &  {\bf 0}  & 7.8 &   0.39 (4)  \\ \cline{2-8}
 & 18 & 1 & 4 & 0.4  & {\bf 0}   & 6.0 & {\bf 0.09}  (1)  \\ \hline
 % &  18 & 1.5 & 4 & 0.6  & {\bf 0}   & 6.68 & {\bf 0}  \\ \hline
 2048 &  6 & 1 & 4 &  0.2 & {\bf 0}  &  8.8& 1.02 (21)  \\ \hline
 \end{tabular}
\end{table}
We observe that the BER after fine-tuning   is always 0 (except when no spreading is applied, $S=1$), hence robustness against fine-tuning is achieved in all the settings.
Regarding robustness against transfer-learning to cat\&horse LSUN, we observe that, when $l = 256$,  BER = 0 is obtained in the setting $S = 12$, $C=1.5$ and $S = 18, C=1$, while in the setting  $S = 12$, $C=1$ we got BER = 1.17\%. This shows that increasing the spreading factor has a similar effect of increasing  the watermark strength $C$. The same behavior is observed with CIFAR-10, as discussed below.
We notice that the TER after transfer-learning is some points larger than the baseline for the  cat\&horse LSUN classification task. This might be due to the fact that transferring the knowledge from GAN detection to  cat\&horse classification is not easy, being them two very different tasks.

%Table \ref{tab:trasferLearningDense}  reports the results of transfer-learning in the  second scenario, where the GTSRB classification model based on %DenseNet is retrained on CIFAR10.
%The results are reported for the various cases of payload, strength and spreading factor.

%
\begin{table}[t]
    %\scriptsize
	%\renewcommand\arraystretch{1.3}
	\setlength{\tabcolsep}{1.2mm}
	\centering
	\caption{Robustness of the ResNet-based CIFAR-10 watermarked model against fine tuning (on 70\% of the dataset) and  transfer learning to  CIFAR-100.
%The TER and BER (\%) after retraining are reported.
The baseline TER for CIFAR-100 is  23.8 \% (Top1) and 6.8\% (Top5).   L = Layers. FT = Fine Tuning. TL = Transfer Learning. The BER values below 0.1\% are highlighted in bold (the number of bit errors is report among brackets).}
    \label{tab:trasferLearningCIFAR10ResNet}
    	\begin{tabular}{c|c|c|c||c|c | c | c | c}
		\hline
\multirow{2}{*}{$l$} &  \multirow{2}{*}{$S$} & \multirow{2}{*}{$C$} &\multirow{2}{*}{No.L} & \multicolumn{2}{c|}{After FT}&  \multicolumn{3}{c}{After TL}\\ \cline{5-9}
& &  &  & TER&  BER & TER-T1 & TER-T5 &   BER\\ \hline
\multirow{4}{*}{256} &  3 & 1  & 2 & 5.4 &  \textbf{0} & 24.8 &  65.8 &  24.21 (62)  \\ \cline{2-9}
 & 25  & 1 &  2  & 5.1 &  \textbf{0} & 25.0 &7.0  & 2.34 (6)\\ \cline{2-9}
  & 25  & 1.5 &  2  & 5.0 &  \textbf{0} & 26.9 & 7.4  & 0.39 (1)\\ \cline{2-9}
 & 50 &  1 &   2 & 5.1 &   \textbf{0} & 24.8 & 6.8 &  \textbf{0} \\ \hline
\multirow{5}{*}{1024}  &  25 & 1 & 2   & 5.2 &  \textbf{0} & 25.0 & 6.9 &1.07 (11) \\ \cline{2-9}
 &  25 & 1.5 & 2   & 5.3 &  \textbf{0} & 27.2 & 7.9 & \textbf{0} \\ \cline{2-9}
&  50 & 1 & 2   &  5.1 &  \textbf{0}  & 24.2 & 7.2 &  \textbf{0.09} (1)\\ \cline{2-9}
&  75 & 1 & 2 & 5.5 &  \textbf{0} & 24.6 &  7.1 &  \textbf{0}\\ \cline{2-9}
&  100 & 1 & 2 &  4.9 &  \textbf{0} & 24.7 & 7.1 & \textbf{0}  \\   \hline
\multirow{2}{*}{2048}  &  75 & 1 & 2  & 5.0 &  \textbf{0} & 24.7 & 6.9 & 0.10 (2) \\ \cline{2-9}
&  100 & 1 & 2  & 5.1 &  \textbf{0} &24.8 & 7.1 & \textbf{0} \\ \hline
\multirow{3}{*}{4096}  &  125 & 1 & 4  & 5.1 &  \textbf{0} & 24.6 & 7.2 &  0.14 (6)\\ \cline{2-9}
&  125 & 1.5 & 4  & 5.3 &  \textbf{0}& 24.3 &  6.4 & \textbf{0} \\ \cline{2-9}
&  150 & 1 & 4  & 5.3 &  \textbf{0} & 24.4 &  7.3 & \textbf{0} \\ \hline
\multirow{2}{*}{8192}  &  150 & 1 & 4  & 5.3 &  \textbf{0}& 24.5 & 6.2 &  0.10 (8)\\ \cline{2-9}
&  200 & 1 & 4  & 5.2 &  \textbf{0}  & 24.6 & 6.9 &  \textbf{0} \\ \hline
\multirow{2}{*}{16384}  &  250  & 1 & 8  & 5.1 &  \textbf{0} & 24.8 & 7.1 & 1.04 (171)\\ \cline{2-9}
&  400 & 1 & 8  & 5.2 &  \textbf{0} & 29.4 & 8.2 & \textbf{0} \\ \hline
    \end{tabular}
\end{table}

\begin{table}[t]
    %\scriptsize
	%\renewcommand\arraystretch{0.7}
	\setlength{\tabcolsep}{1.3mm}
	\centering
	\caption{Performance  and robustness of DenseNet-based CIFAR-10 watermarked model against  fine-tuning (on 70\% of the dataset) and transfer learning to GTSRB.
The baseline TER for CIFAR-10 is 4.7\%,
%The baseline TER-Top-3 for CIFAR-10 is  5.62\%.
%The baseline TER for CIFAR-10 is  19.67\% (Top-1) and  5.62\% (Top-3).
while for  GTSRB is  3.7\%. L = Layers. FT = Fine Tuning. TL = Transfer Learning.}
%The BER values below 0.1\% are highlighted in bold. }
    \label{tab:trasferLearningCIFAR10DenseNet}
	\begin{tabular}{c|c|c|c||c || c| c |c |c }
		\hline
\multirow{2}{*}{$l$} & \multirow{2}{*}{$S$} & \multirow{2}{*}{$C$} &\multirow{2}{*}{No.L} &  \multirow{2}{*}{TER} & \multicolumn{2}{c|}{After FT} &  \multicolumn{2}{c}{After TL}  \\ \cline{6-9}
& & & &   & TER &  BER & TER &  BER  \\ \hline
{256} & 50 &1  & 2 &  4.4  & 4.3 & \textbf{0}  & 4.2  & \textbf{0}\\  \hline
{1024} & 75 &1  & 2 &  4.8  & 4.8 & \textbf{0}  & 3.4  & \textbf{0}\\  \hline
 2048  & 100 &  1& 4     & 4.5 & 4.6 &  \textbf{0} & 4.1  & \textbf{0}\\  \hline
 %& 30 & 1 &   2  & ?  & 2.51 & 2.24& 0 & 7.28  &  0.39 \\ \cline{2-10}
   4096  & 150 & 1 & 6 &  5.0 & 4.5 & \textbf{0} & 3.3 & \textbf{0}\\   \hline
    8192  &  200 & 1 & 12 &  4.3 & 4.4 & \textbf{0} & 3.2 & \textbf{0}\\   \hline
    \end{tabular}
\end{table}

The results obtained by retraining the CIFAR-10 classifiers are reported in Table  \ref{tab:trasferLearningCIFAR10ResNet} and \ref{tab:trasferLearningCIFAR10DenseNet}.
Specifically, in Table  \ref{tab:trasferLearningCIFAR10ResNet} we report the robustness performance of the ResNet-based CIFAR-10 watermarked models against both fine-tuning and transfer-learning to CIFAR-100. The TER-Top1 and TER-Top5  (indicated as TER-T1 and TER-T5 in the tables) obtained after transfer learning are aligned with those achieved by the baseline model trained on the CIFAR-100 task from scratch.
Table \ref{tab:trasferLearningCIFAR10DenseNet} instead reports the  performance of the DenseNet CIFAR-10 watermarked models.
% for some settings/selected cases, where zero BER is achieved after retraining.
Fine-tuning  and transfer-learning to GTSRB are considered for the retraining experiments.

By looking at the results in Table \ref{tab:trasferLearningCIFAR10ResNet}, we see that, even in this case, robustness against fine-tuning is always achieved in all the settings with BER = 0. The use of a large spreading factor is fundamental to improve the robustness in the more challenging scenario of transfer learning.
A very large $S$ has to be considered to achieve robustness in this case.
Increasing the strength $C$ would obviously also help to improve the robustness, yet at the price of a reduced indistinguishability of the watermark, when $C$ becomes large.
The results reported in the table  show that, as long as $S$ is large enough, robustness can always be achieved with BER equal to 0 in all the settings.
%
%%This explains why a larger $S$  has to be used to achieve the same level of performance.
%Enlarging the strength $C$ would obviously also help, yet at the price of a reduced indistinguishability of the watermark, when $C$ is large, e.g. $C > 2/3$.

%\BTnew{We also observe that, when the payload increases, a larger spreading is necessary to achieve  BER equal to 0 after transfer learning. A possible explanation is the following: when a larger percentage of weights is dedicated to bringing the watermark information,
%when  a larger percentage of weights is  responsible for the watermark information in the watermarked model,
%during the retraining on the new task (with all weights  'updatable') the weights are modified by larger extent by the optimization procedure.
%%to get a good solution on the new task
%}
%\MB{This explanation does not explain much, should we remove it?}

Similar results are obtained for the DenseNet-based CIFAR-10 watermarked models. The results in Table \ref{tab:trasferLearningCIFAR10DenseNet}
are  reported for some selected settings, where BER=0 is achieved after retraining.
Since the  layers that we selected for the embedding in the DenseNet case have a lower number of weights, for  high payloads, a larger number of layers have to be considered to avoid layer saturation,
% of the embedded layers,
thus compromising the unobtrusiveness.
%For instance, for $l = 8192$ (with $C = 1$, $S = 200$),
%%in order to avoid unobtrusiveness,
%we had to consider 12 layers for the embedding, with a average percentage of watermarked weights   $\bar{p}^m_k$ of about  $73$\%.
For instance, for $l = 4096$ (with $C = 1$, $S = 150$), we had to consider 6 layers for the embedding, with a average percentage of watermarked weights   $\bar{p}^m_k$ of about  $52$\%.

Finally,  we observe that, according to our experiments, as long as the percentage of embedding weights in each layer is not too large,
embedding the watermark in more layers with a lower percentage of hosting weights does not give any significant advantage with respect to embedding the watermark in less layers with a larger percentage of hosting weights.

\subsection{Comparison with the state-of-the-art}
\label{subsec.SOA}

In this section, we compare our method with the white-box multi-bit DNN watermarking methods by Uchida et al. \cite{uchida2017embedding}, Li et al. \cite{yue21jins} and the one by Liu et al. \cite{liu2021watermarking}.
In all the cases, the methods have been implemented by using the code released by the authors.
%
%\BTcomm{Chiaramente andare a confrontarci con Uchida e Yue lascia un po' il tempo che trova e serve piu' che altro per fare numero.....ad andare a confrontarci con i metodi basati su passports non ci passa piu'. }
%We applied them to watermark the DenseNet-40 architecture.

For \cite{uchida2017embedding} and \cite{yue21jins}, by following the setting in \cite{yue21jins}, training is carried out by  embedding the watermark in the first convolutional layer of the second dense block of a DenseNet169 architecture.
%As with our method, fine-tuning is  performed by training the network for additional 10 epochs on 70\% of the same dataset.
%The transfer-learning scenario is also considered, re-training the watermarked model on a different tasks and dataset.
%
% (for the number of epochs necessary to get a good accuracy, that is around 20)
%
For  \cite{liu2021watermarking}, following the paper and the implementation in the authors' repository, the watermark is embedded in the first convolutional layer  of a ResNet18 architecture.
%\BTcomm{Il fatto che per Yue e Uchida si marchi DenseNet e per il  nuovo metodo ResNet  rende i metodi non del tutto direttamente confrontabili.  D'altra parte: 1) il codice che ci ha dato Yue cosi' come e' funziona solo su alcune reti con le quali ha lavorato, e non su altre con diverse tipologie di layer. Bisognerebbe perderci un pochino di tempo 2) il codice del metodo greedy di Liu e' scritto e ottimizzato per lavorare su ResNet18 e testato da loro su questa rete. In questo caso il lavoro di estensione sarebbe piu' semplice perche' il codice e' scritto bene e ben organizzato, ma richiede ulteriore lavoro ed e' un salto nel buio. Ad un caso si puo' togliere il paragrafo sopra e non specificare niente (l'unica nota e' che il codice di Yue non e' public released). Gran parte dei paper omettono questi dettagli. Ad ogni modo, nei confronti di Liu che e' il vero benchmark, il confronto e' fair perche' si considera pari pari la stessa architettura. Il confronto con Uchida e Yue lascia un po' il tempo che trova, visto anche il gap di prestazioni che c'e', e servono piu' che altro per fare numero  (e comunque i risultati che abbiamo noi con questa stessa rete sono migliori - vedi tabella sopra). Ad andare a confrontarci con i metodi basati su passports, che funzionano meglio almeno quanto a robustness, non ci passava piu'. }
%\MB{Va bene}

Table \ref{tab:comparison} reports the comparison of the results obtained by these methods with those of our scheme, when the networks are trained on CIFAR-10, for various values of the payload, namely $l = 128$, $256$ and $1024$.
%Results include TER, BER of the watermarked models and robustness against fine-tuning  and  transfer-learning.
For a fair comparison, as for our method, fine-tuning of the state-of-the-art methods is carried out by training the watermarked networks for 10 additional epochs on 70\% of the original dataset with the same learning rate, and transfer learning is performed for the same number of epochs, again with the same learning rate.

With regard to our method, the performance of the ResNet-based models are reported (those achieved on DenseNet being very similar), for the following  settings  $l$-$C$-$S$: 128-1-50, 256-1-50 and 1024-1-75.
%In all the cases, the watermark is embedded with a 2-layer embedding ?????\TODO{Check}.
%For CIFAR-100, the TER-Top5 is reported in the tables.
%%%\AC{ Andrea: for CIFAR cases (training, finetuning, transfer learning) I reported TER(\%)-Top3, while TER(\%)-Top1 is commented out.}
%
%%With regard to our method, both for  GTSRB and CIFAR-10,  the performance are reported for the 2 layers embedding setting except for the case with payload 128, in which case the watermark is embedded in a single-layer. For the case $l = 256$ and $l = 1024$,  results are reported for one among the best watermark settings with $C=1$  in Table \ref{tab:trasferLearningDense} and \ref{tab:trasferLearningDenseCIFAR}.
%
By looking at the table we observe that, while the unobtrusiveness is good for all the methods\footnote{While for the proposed method the BER of the watermarked model is 0 by construction,  with the methods that embed the watermark by using a properly modified loss function, the BER  after training can be different than 0.} the robustness of the methods in \cite{uchida2017embedding} and \cite{yue21jins} against retraining is poor. Some robustness against fine-tuning can be achieved by the method in \cite{uchida2017embedding} when the payload is low (BER= 0 with $l = 128$ and 5\% when $l = 256$), while for the method in   \cite{yue21jins}  the BER is above 20\% in all the settings. In both cases,  the TER obtained after fine-tuning is 2/3\% worse.
however, both \cite{yue21jins} and \cite{uchida2017embedding} fail against transfer learning, also in the easier case of transfer to the GTSRB (only 20 retraining epochs, in contrast to the 200 epochs of the CIFAR-100 case).
This confirms the superiority of the proposed scheme, that can achieve strong robustness against both fine-tuning and retraining for all the payloads considered.

The robustness of the method by Liu et al. \cite{liu2021watermarking} is comparable to those achieved by our method. However, our approach overcomes \cite{liu2021watermarking} in terms of invisibility of the watermark and payload. In fact,  with the
approach in  \cite{liu2021watermarking},  robustness is indirectly achieved
by  increasing the strength of the weights  responsible of carrying out the watermark information in  the first convolutional layer (the weights of this layer, in fact, tend to have a larger strengths with respect to the weights in the other layers) by large extent, so that they can survive to several retraining iterations.
%
%Therefore, the watermark is always embedded in the first layer.
However, doing so, the secrecy requirement is not satisfied.
Figure \ref{fig.greedy} shows the distribution of the weights in the embedded layer (first layer) for the non-watermarked (baseline) model and the  model watermarked
with the algorithm in \cite{liu2021watermarking} with payload $l=$256.
The histogram in the watermarked case shows visible peaks for large magnitudes, that  revel the presence of the watermark inside the model and the position of the  weights carrying the watermark information.

\begin{figure}[]
	\centering
\includegraphics[width=0.462\columnwidth]{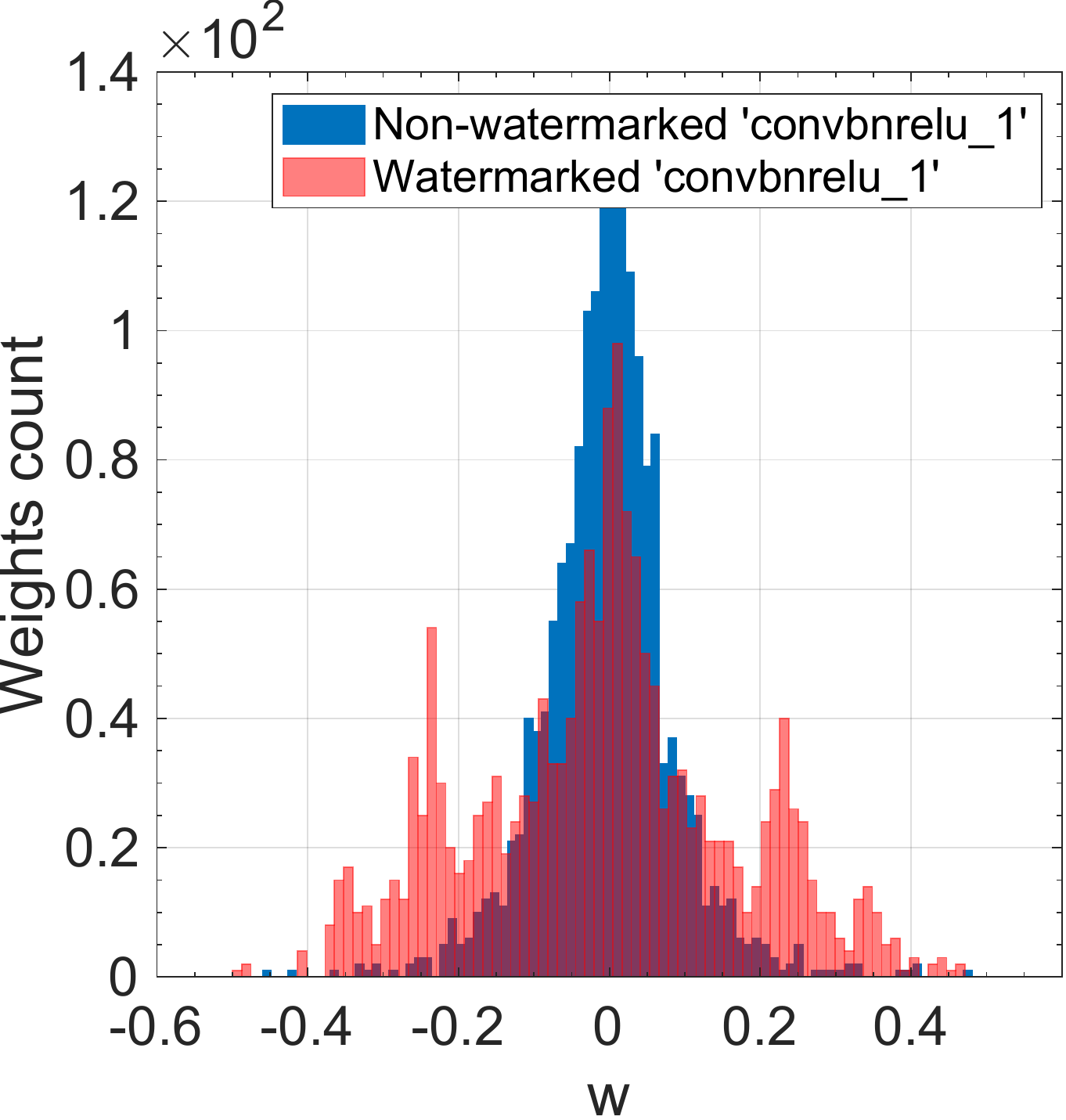}
\includegraphics[width=0.475\columnwidth]{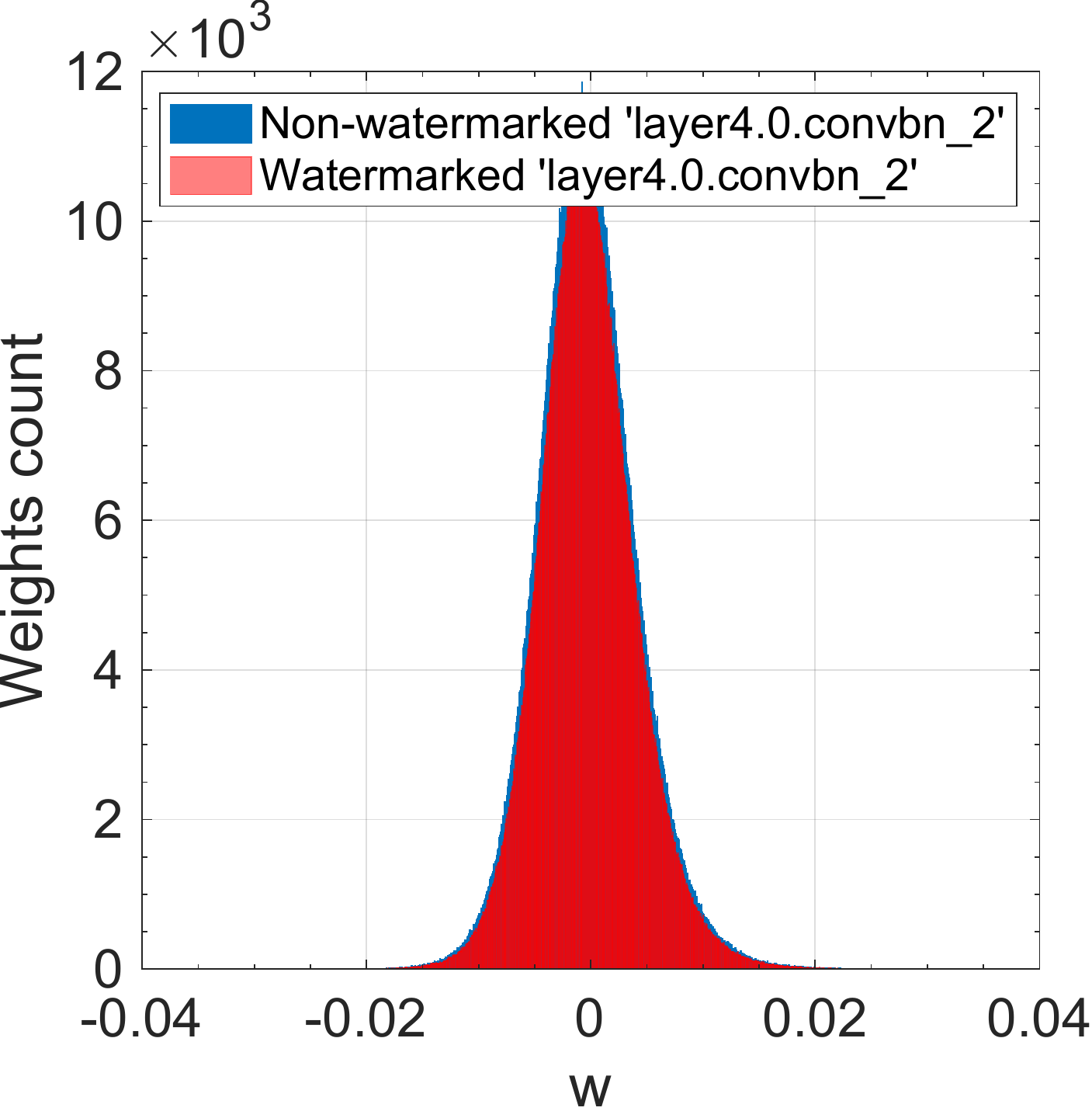}
	\caption{Distribution of  the weights  in the embedding  layer for the non-watermarked and  watermarked model for the method in \cite{liu2021watermarking} and for our method in the setting ResNet18-256-50.
}
	\label{fig.greedy}
\end{figure}

\begin{figure}[]
	\centering
    \includegraphics[width = 0.48\columnwidth]{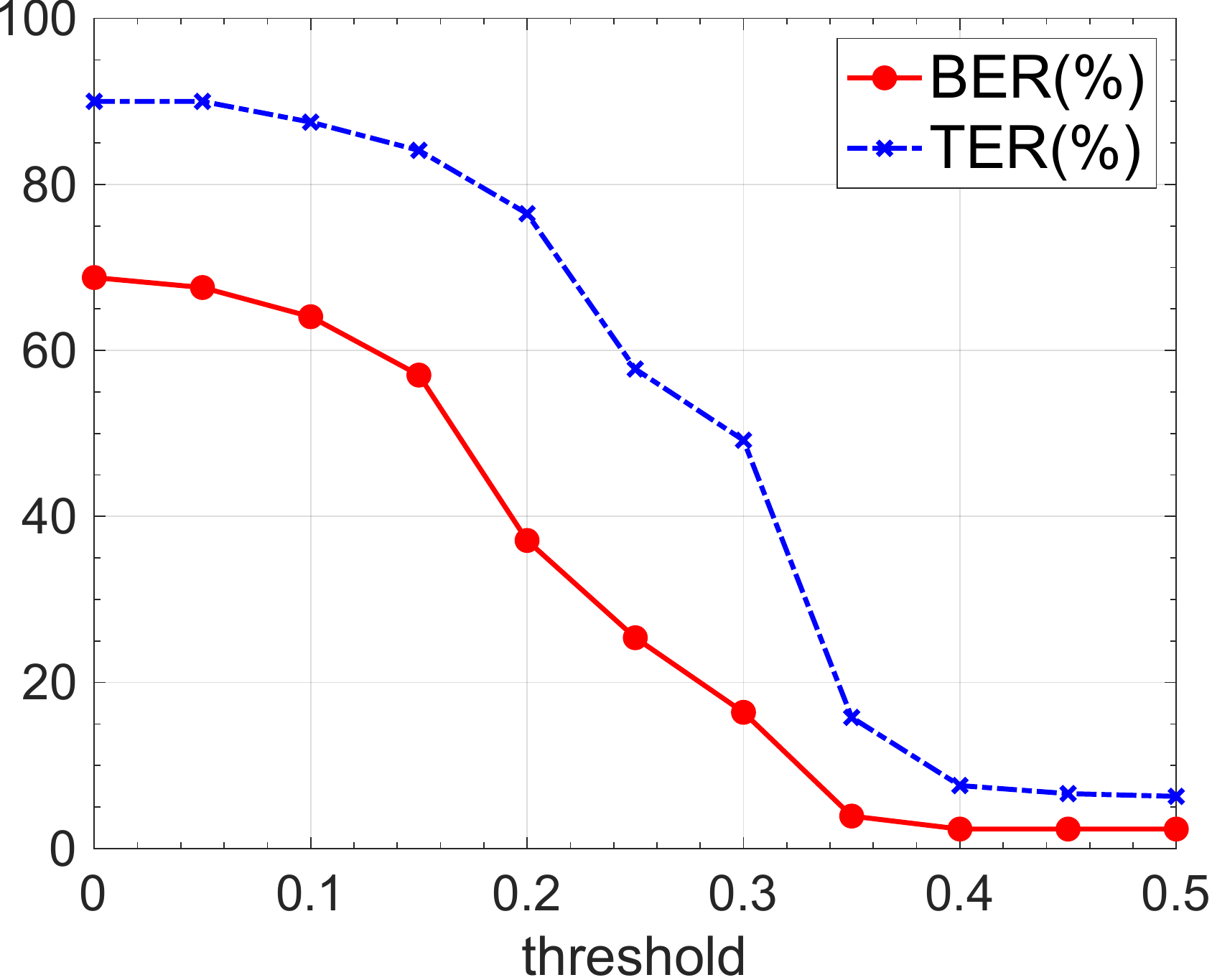}
    \includegraphics[width = 0.48\columnwidth]{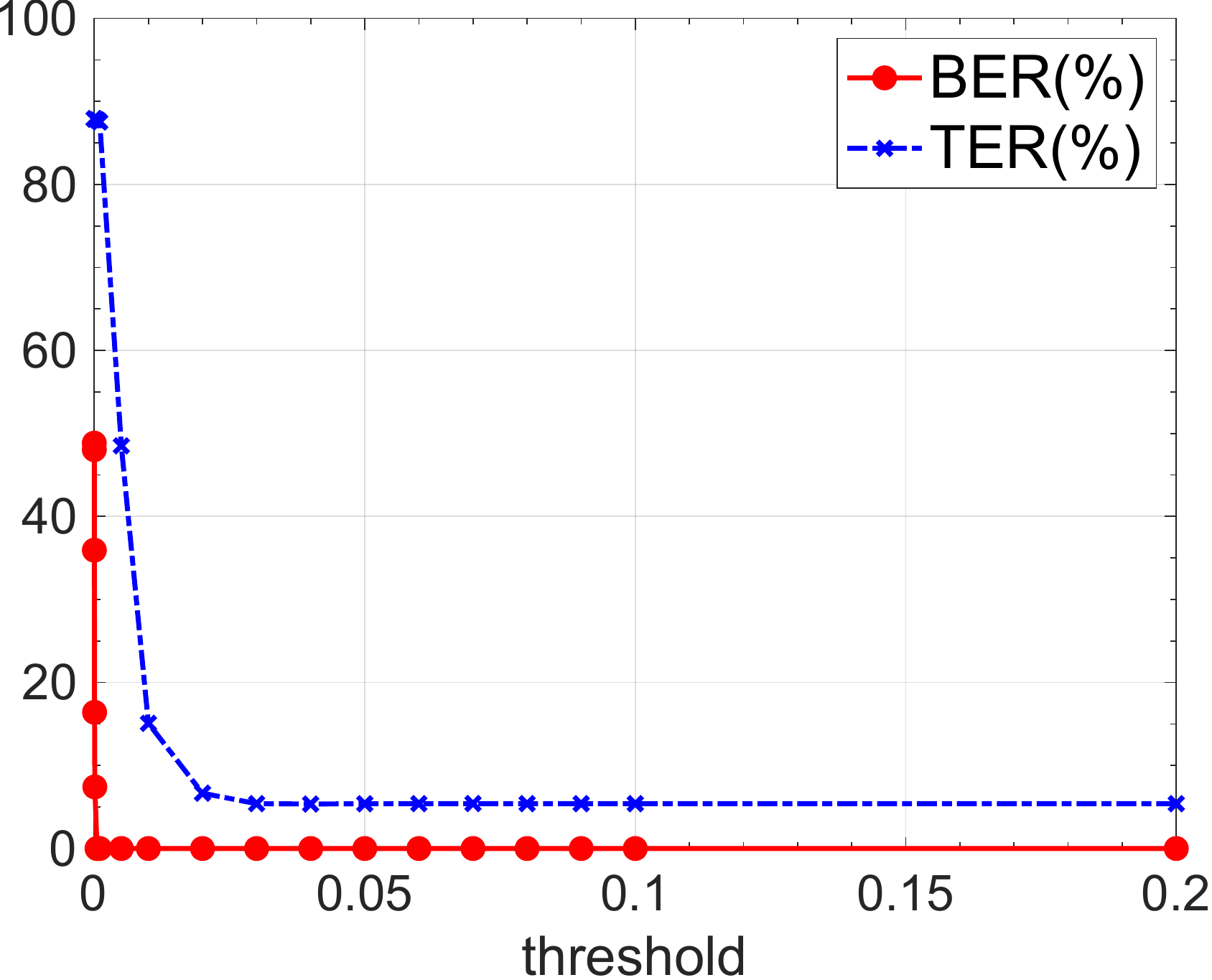}
	\caption{TER and BER as a function of the cut-off threshold $T$. Top line: the networks weights above the threshold in the embedding layer(s) are set to 0 in the model watermarked by the method in  \cite{liu2021watermarking} (left) and our method (right).
%Bottom line: all networks weights above the threshold are set to 0 in the model watermarked with our method.
%
%\AC{enlarge caption. Manca il plot dove il cut per il nostro metodo si fa a tutti i layer della rete.}
}
	\label{fig.greedy-cut}
\end{figure}

To further prove the better secrecy of our method with respect to\cite{liu2021watermarking}, we implemented a simple attack strategy whereby an attacker, knowing that the watermark information is preferably embedded in large weights, sets to zero all the weights of the watermarked layer(s) larger than a certain threshold.
Figure \ref{fig.greedy-cut} shows the behavior of the TER and the BER, as a function of the cut-off threshold,  for the method in \cite{liu2021watermarking} (a) and our method (b)-(c), for the same ResNet18-based CIFAR-10 model watermarked with $l = 256$ (the setting ResNet-CIFAR10-256-1-50 is considered in our case).
%We checked that,
%For the method in \cite{liu2021watermarking}  by cutting for instance at ?? \AC{che cut applicasti in questo test che facesti?}
We found that, by cutting at 0.15
%\AC{Per la rete base di Greedy residuals (payload 256, spread 253), con cutoff = 0.1 si ottenevano TER=87.1\% e BER=56.6\%, che tornava nel giro di una epoca a circa TER=6\%, ovviamente a parità di BER}
the weights of the first layer of the model in \cite{liu2021watermarking}, and retraining the model for some iterations (around a quarter of Epoch) the network functionality can be restored, even when retraining is performed on a subset of the data.
%
%Instead, when the same procedure is applied to our method, for the model  with $S = ??$ ....see Figiure ?? (b)[DUE PLOT - the embedded layer is unknown (corresponding to a secret information), the scenario where all the weights of the network are cut is more realistic].
This is not the case with our method, see Figure \ref{fig.greedy-cut}(b)\footnote{We stress that, since the embedded layers are secret with our method -  determined by the secret key $K$ ($\Omega$) -  assuming that their are known we are actually favoring the attack.}.
%Figure \ref{fig.greedy-cut}(b)-(c)
We see that, in order to impair the watermark, a very small cut-off must be used, completely destroying the functionality of the network and making it unusable. We verified that recovering the accuracy, in this case, requires a retraining effort equal to training the network from scratch.
Although the above analysis focuses on the case $l = 256$, similar considerations can be made for other payloads.
%
%

%With regard to the payload,
We also verified that our method can embed a larger payload with respect to \cite{liu2021watermarking}.
%Since in \cite{liu2021watermarking} the watermark is  embedded in the first layer (where the magnitude of the weights is larger),  the  maximum payload depends on the dimensionality of this layer \MB{Is this something you disciovered from the implementation or is clearly written in the paper?}\BTcomm{No. The paper presents it in general, but the results they report refer to this case (in the way they work, embedding the watermark in other layers with smaller variance I do not expect their method do not work). However, probably here the best thing is to remove this observation.} \MB{This is the only critical point I see. If the paper is reviewed by the authors of \cite{liu2021watermarking}, they can cheat and claim that our comparison is not fair since their method can also by used with different layers } \BTcomm{I see. However, in the way their method work, it is not obvious how to extend it method to the multi-layer case, and there are some design choice to make (e.g., how to build their matrix that gives the association of the biggest per-row  weights  to the bits,...). I do not think it is our job (at least, I do not see this effort done in other papers....). Moreover,the code they released is written to work only when the embedding is done in the first layer, the parameter setting they consider only work in this case, everything would have to be readapted to make it work with a different layer (let alone the multi-layer).... I would just remove the observation.}.
Table \ref{tab:comparison-Greedy} reports the unobtrusiveness and robustness results for large payloads obtained watermarking the same ResNet18 architecture with the two methods. We see that, the method in \cite{liu2021watermarking} is no longer able to successfully embed the watermark when $l= 4096$, with the BER going above 25\% (and remaining similar after retraining).
With out method instead, we are able to embed more than 16.000 bits without affecting the unobtrusiveness, and satisfying the robustness requirement (the BER remains 0  after both fine tuning and transfer learning).

\begin{table}

\centering

	\caption{Comparison with existing methods.  TER and BER are reported (FT = fine tuning. TL = transfer learning). Results refer to the case of watermarked models for CIFAR-10. The baseline TER for the TL tasks is 3.6\%  for GTSRB and  23.8 \% (Top1) and 6.8\% (Top5) for CIFAR-100. The BER results after retraining are highlighted in bold.
%\BTcomm{Se si considerasse l'augmentation sarebbe 2.62\%. Vedi commento nel testo.}. \BTcomm{Per il nsotro metodo ho cambiato una riga aggiungendo il risultato in uno dei nuovi setting che ho inserito nell'altra tabella.}
}

	%\footnotesize

		%\renewcommand\arraystretch{0.8}

		%\setlength{\tabcolsep}{0.55mm}

	\label{tab:comparison}

	%	\begin{tabular*}{\tblwidth}{@{} CCCCCCCCC@{} }
	%\setlength{\tabcolsep}{1.3mm}
	\setlength{\tabcolsep}{0.6mm}{
 %\scriptsize

		\begin{tabular}{l c|c|c|c||c|c|c|c|c|c|c|}

			%	\toprule

			\hline

				% \multirow{2}{*}{Method} &  & \multirow{3}{*}{$l$}  &  \multicolumn{6}{c|}{ Uchida et al. \cite{uchida2017embedding}}
%& \multicolumn{6}{c|}{Li et al. \cite{yue21jins}} & \multicolumn{7}{c}{Proposed }\\ \cline{3-21}

      \multirow{2}{*}{} &  &   \multirow{2}{*}{$l$} & \multirow{2}{*}{TER} & \multirow{2}{*}{BER}  & \multicolumn{2}{c|}{ FT} &   \multicolumn{3}{c|}{TL (CIFAR-100)} &    \multicolumn{2}{c|}{TL (GTSRB)}\\  \cline{6-12}
            &   &   & &   & TER & BER & TER-T1 & TER-T5 & BER &   TER & BER   \\  \hline
            \multirow{3}{*}{\cite{uchida2017embedding}} &    & 128   & 6.6 & 0 & 8.1 & {\bf 0} &  27.5& 6.8& \bf{42.96} & 4.4& \bf{51.56}\\  \cline{3-12}
            &   &  256   & 6.7  & 0  & 7.7 & {\bf  5.08} & 28.2 & 7.5& {\bf 44.14}& 4.6& {\bf 55.08}\\  \cline{3-12}
            &  &  1024   &  7.1 & 4.59 &  7.8&  {\bf 34.40}& 28.6& 7.7& {\bf 47.46}& 4.2& {\bf 49.60}\\ \hline
              \multirow{3}{*}{\cite{yue21jins}} &    & 128   & 7.0 &  0&  7.6& {\bf 28.12}& 28.3 & 7.6& {\bf 50.00} & 4.6& \bf{47.19}\\  \cline{3-12}
            &   &  256   &  6.6 &  0 & 8.2 & {\bf 23.04} &  27.6& 7.4&  {\bf 50.78} & 4.2& \bf{50.01}\\  \cline{3-12}
                        &  &  1024   & 6.8 & 0 & 7.6 &  {\bf 30.76}  & 27.3 & 7.0& {\bf 50.68}& 3.7& \bf{50.19} \\ \hline
                                                            \multirow{3}{*}{\cite{liu2021watermarking}} &    & 128   & 5.2& 0& 5.2& \bf{0}& 27.4& 8.0 & \bf{0} & 4.6& \bf{0}\\  \cline{3-12}
            &   &  256   & 5.3&0 & 5.0& \bf{0}& 27.0& 8.1& \bf{0}& 4.5& \bf{0}\\  \cline{3-12}
                        &  &  1024   & 5.3 & 0& 5.2& \bf{0.29}& 27.3& 8.1& \bf{4.10}& 4.6 & \bf{1.70}\\ \hline
                                                                                    \multirow{3}{*}{Prop} &  & 128   & 5.0 & 0 & 5.3  & {\bf 0} & 24.9 & 6.1 &  {\bf 0} & 4.7 & {\bf 0} \\  \cline{3-12}
            &  &  256   & 5.2 &  0 & 5.1 & {\bf 0} & 24.8  & 6.8 & {\bf 0} & 4.6 & {\bf 0} \\  \cline{3-12}
                        &  &  1024   & 6.5 & 0 & 5.5 & {\bf 0} & 24.6 & 7.1 & {\bf 0} & 3.9 &  {\bf 0}\\ \hline

			\hline
	\end{tabular}
}
\end{table}

\begin{table*}

\centering

	\caption{Comparison with the method in \cite{liu2021watermarking} for large payloads.
%\BTcomm{Se si considerasse l'augmentation sarebbe 2.62\%. Vedi commento nel testo.}. \BTcomm{Per il nsotro metodo ho cambiato una riga aggiungendo il risultato in uno dei nuovi setting che ho inserito nell'altra tabella.}
}

	%\footnotesize

		%\renewcommand\arraystretch{0.8}

		%\setlength{\tabcolsep}{0.55mm}

	\label{tab:comparison-Greedy}

	%	\begin{tabular*}{\tblwidth}{@{} CCCCCCCCC@{} }

	\setlength{\tabcolsep}{0.6mm}{
 %\scriptsize

		\begin{tabular}{c|c|c|c|c|c|c|c|c|c|c|c|c|c|c|c|c|c|c|c}

			%	\toprule

			\hline

				 \multirow{3}{*}{$l$}  &  \multicolumn{9}{c|}{Liu et al.  \cite{liu2021watermarking}}  &
 \multicolumn{10}{c}{Proposed} \\ \cline{2-20}
      &   \multirow{2}{*}{TER} & \multirow{2}{*}{BER}  & \multicolumn{2}{c|}{FT} &   \multicolumn{3}{c|}{ TL (CIFAR-100)} &  \multicolumn{2}{c|}{ TL (GTSRB)}  &  \multirow{2}{*}{($S$,$C$)} & \multirow{2}{*}{TER} & \multirow{2}{*}{BER}  & \multicolumn{2}{c|}{ FT} &   \multicolumn{3}{c|}{ TL (CIFAR-100)} &  \multicolumn{2}{c}{ TL (GTSRB)}  \\  \cline{4-10} \cline{14-20}
      & &  & TER & BER & TER & TER &  BER  & TER &  BER & &  &  & TER & BER & TER & TER &  BER  & TER &  BER  \\ \hline
      2048 & 5.1& 4.7& 5.0& \bf{8.20}& 26.9& 8.1& \bf{11.04}& 5.4& \bf{8.92}& (100,1)  &  5.0&  0 &  5.1 &  \textbf{0} &24.8 & 7.1 & \textbf{0} &  4.0 & {\bf 0} \\ \hline
           4096 & 4.9& 25.85& 4.8& \bf{26.22}& 26.8& 8.1& \bf{26.95}& 5.6& \bf{26.64}& (150,1)   & 5.3  & 0 & 5.3 &  \textbf{0} & 24.4 &  7.3 & \textbf{0} & 4.7 & {\bf 0} \\ \hline
           8192 & 5.2& 33.65& 5.3& \bf{33.62}& 26.5& 8.3& \bf{33.69}& 4.4& \bf{33.61}& (200,1)  & 5.5 & 0 &  5.2 &  \textbf{0}  & 24.6 & 6.9 &  \textbf{0} & 3.8 & {\bf 0} \\ \hline
            16384 & 5.7 & 35.19& 5.4& \bf{35.17}& 26.9& 8.3& \bf{35.31}& 4.7& \bf{35.27}& (400,1)  & 4.8 & 0 & 5.2 &  \textbf{0} & 29.4 & 8.2 & \textbf{0} & 4.1 & {\bf 0} \\
\hline
	\end{tabular}
}
\end{table*}

\section{Concluding  Remarks}
\label{sec.conc}

%\MB{Thesummary should be changed to reflect the different approach to tell our algorithm we are using now}

%\MB{I have always thought that the concluding section should not be just a summary of what we have done in the paper, rather it should stress the most important lessons we have learnt and stress the most important results now that the reader knows all the details \dots I admit that I am not eager to rewrite this section from scratch, so unless you want to try to make this section more original, you can leave it as is.}

By the light of a new interpretation of the watermark trade-off triangle in the case of DNN watermarking, we designed an effective white-box, multi-bit,  watermarking algorithm reaching a good tradeoff among the (revisited) requirements that any DNN watermarking scheme should satisfy, achieving an extremely large payload and outstanding robustness.
%
%We proposed a white-box DNN watermarking algorithm  for multi-bit embedding that exploits a new information embedding strategy to improve the robustness of the watermark against network modifications and  re-use.
%
The watermark
message is spread across a number of fixed weights (watermarked weights), whose
position depends on a secret key, and whose value is set prior to training and left unchanged during the training procedure. The values of the other weights responsible for the accomplishment of the network primary task are then determined in a watermark-dependent manner.
The strength of the watermarked weights is made large enough to survive network modification. The secrecy of the watermark is ensured by
optimizing the distribution of the watermarked weights in such a way to minimize the KL distance between watermarked and non-watermarked weights, for a given strength of the watermark.
%
%The distribution of the watermarked weights is theoretically optimized in such a way to minimize the distinguishability between watermarked and non-watermarked weights\BTnew{, thus satisfying the invisibility requirement}.
%
Experimental results show that the proposed algorithm can achieve very large payloads while being robust against network modifications and re-use.
Robustness can also be achieved in the challenging transfer learning
scenario, with payloads that are out of reach of state of the art methods.

In future works, we will  focus on a comprehensive  analysis of the security of the proposed scheme against informed attackers, that is, attackers  aware of the
presence of the watermark, who try to erase the watermark via removal attacks, by embedding a new watermark (overwriting attacks) or by unveiling the watermark secret key, e.g., by analyzing the importance that the network weights have on the classification performance of the network.
%\BTcomm{My concern is that a reviewer can complain that the method  is not secure unless further strategies or actions are taken.}
%
An interesting direction to further improve the robustness of the proposed DNN watermarking system would be to apply channel-coding to the informative message  \cite{perez2001approaching}, investigating the kind of  codes that are most suitable in the DNN watermarking scenario adopted by our method.

%\BTcomm{We should also mention that, since the position watermarked weights can not be revealed by looking at the weight distribution, we can say that the scheme also guarantees a good level of  security (thanks to the secrecy of the key). However, further experiments can be evaluated to assess the security against informed attackers (i.e.,  in adversarial setting wherein the users explicitly aim at removing the watermark from the model), e.g. oracle attacks or overwriting attacks.}

%\section*{Acknowledgements}
%%
%This work has been partially supported by.....\BTcomm{Direi che non abbiamo niente} \MB{Direi di no}
%%a research sponsored by DARPA and Air Force Research Laboratory (AFRL) under agreement number FA8750-16-2-0173. The U.S. Government is authorised to reproduce and distribute reprints for Governmental purposes notwithstanding any copyright notation thereon. The views and conclusions contained herein are those of the authors and should not be interpreted as necessarily representing the official policies or endorsements, either expressed or implied, of DARPA and Air Force Research Laboratory (AFRL) or the U.S. Government.

%\TODO{nella biblio controlla che non ci siano reference doppie (ho copiato da due .bib) }
%
%\TODO{problema con reference [33] vuota}
%
%\appendix
%
%%\section*{Appendix}
%
%\subsection*{Geometric Transformation Estimation.}

%
% REFERENCES
%
\ifCLASSOPTIONcaptionsoff
\newpage
\fi
\bibliographystyle{IEEEtran}
\bibliography{mybibliography}

\end{document}